\documentclass{article} %
\usepackage{configuration/iclr2026_conference,times}

\usepackage{hyperref}
\usepackage{url}

\usepackage{transparent}

\usepackage{amsmath,amssymb,amsthm}

\providecommand{\norm}[1]{\left\lVert#1\right\rVert}

\providecommand{\R}{\mathbb{R}} %

\providecommand{\E}{{\mathbb E}}
\providecommand{\E}[1]{{\mathbb E}\left.#1\right. }        %
\providecommand{\Eb}[1]{{\mathbb E}\left[#1\right] }       %
\providecommand{\EEb}[2]{{\mathbb E}_{#1}\left[#2\right] } %

\providecommand{\dm}{\mathrm{d}}

\renewcommand{\aa}{\mathbf{a}}
\providecommand{\bb}{\mathbf{b}}
\providecommand{\cc}{\mathbf{c}}

\providecommand{\rr}{\mathbf{r}}
\renewcommand{\ss}{\mathbf{s}}

\providecommand{\vv}{\mathbf{v}}

\providecommand{\xx}{\mathbf{x}}
\providecommand{\yy}{\mathbf{y}}
\providecommand{\zz}{\mathbf{z}}

\providecommand{\mepsilon}{\boldsymbol{\epsilon}}
\providecommand{\mphi}{\boldsymbol{\phi}}

\providecommand{\mtheta}{\boldsymbol{\theta}}

\providecommand{\mf}{\boldsymbol{f}}
\providecommand{\mv}{\boldsymbol{v}}
\providecommand{\mg}{\boldsymbol{g}}
\providecommand{\mmF}{\boldsymbol{F}}

\providecommand{\cF}{\mathcal{F}}

\providecommand{\cL}{\mathcal{L}}

\newenvironment{talign*}
{\let\displaystyle\textstyle\csname align*\endcsname}
{\endalign}

\usepackage[utf8]{inputenc}         %
\usepackage[T1]{fontenc}            %
\usepackage{url}                    %
\usepackage{booktabs}               %
\usepackage{amsfonts}               %
\usepackage{nicefrac}               %
\usepackage{microtype}              %
\usepackage{xcolor}                 %
\usepackage{algorithm}
\usepackage{algorithmic}
\usepackage{placeins}
\usepackage{graphicx}
\usepackage{subcaption}
\usepackage[flushleft]{threeparttable}
\usepackage{float}
\usepackage{multirow}
\usepackage{xspace}
\usepackage{enumitem}
\usepackage[font=small]{caption}
\usepackage{autobreak}
\usepackage{sidecap}
\usepackage{wrapfig}
\usepackage{bbding}
\usepackage[toc, page, header]{appendix}
\usepackage{tikz}
\usetikzlibrary{arrows.meta,positioning,calc}
\usepackage{xcolor}
\usepackage{pifont}
\usepackage{mdframed}
\usepackage{colortbl}
\usepackage{tcolorbox}
\tcbuselibrary{skins, breakable, theorems}

\definecolor{coral}{RGB}{255,127,80}
\definecolor{darkgreen}{RGB}{0,100,0}
\definecolor{darkyellow}{RGB}{204,153,0}
\definecolor{salmon}{RGB}{250,128,114}
\definecolor{darkred}{RGB}{150,0,0}
\definecolor{darkblue}{RGB}{0,70,150}
\definecolor{algorithmgray}{RGB}{160,160,160}
\definecolor{statusorange}{RGB}{204,112,0}

\hypersetup{
  colorlinks=true,
  linkcolor=darkred,
  citecolor=darkred,
  urlcolor=darkred
}

\newcommand{\darkredtext}[1]{{\color{darkred}#1}}
\newcommand{\statusyes}[1]{\textbf{\textcolor{darkgreen}{#1}}}
\newcommand{\statusno}[1]{\textbf{\textcolor{darkred}{#1}}}
\newcommand{\statusother}[1]{\textbf{\textcolor{statusorange}{#1}}}

\newcommand{\secref}[1]{\hyperref[#1]{\darkredtext{Sec.~\ref*{#1}}}}
\newcommand{\thmref}[1]{\hyperref[#1]{\darkredtext{Thm.~\ref*{#1}}}}
\newcommand{\defref}[1]{\hyperref[#1]{\darkredtext{Def.~\ref*{#1}}}}
\newcommand{\propref}[1]{\hyperref[#1]{\darkredtext{Prop.~\ref*{#1}}}}
\newcommand{\assumpref}[1]{\hyperref[#1]{\darkredtext{Assump.~\ref*{#1}}}}
\newcommand{\remarkref}[1]{\hyperref[#1]{\darkredtext{Rem.~\ref*{#1}}}}
\newcommand{\hypref}[1]{\hyperref[#1]{\darkredtext{Hyp.~\ref*{#1}}}}
\newcommand{\conjref}[1]{\hyperref[#1]{\darkredtext{Conj.~\ref*{#1}}}}
\newcommand{\lemref}[1]{\hyperref[#1]{\darkredtext{Lem.~\ref*{#1}}}}
\newcommand{\corref}[1]{\hyperref[#1]{\darkredtext{Cor.~\ref*{#1}}}}
\newcommand{\noteref}[1]{\hyperref[#1]{\darkredtext{Nota.~\ref*{#1}}}}
\newcommand{\claimref}[1]{\hyperref[#1]{\darkredtext{Clm.~\ref*{#1}}}}
\newcommand{\obsref}[1]{\hyperref[#1]{\darkredtext{Obs.~\ref*{#1}}}}
\newcommand{\algref}[1]{\hyperref[#1]{\darkredtext{Alg.~\ref*{#1}}}}
\newcommand{\figref}[1]{\hyperref[#1]{\darkredtext{Fig.~\ref*{#1}}}}
\newcommand{\tabref}[1]{\hyperref[#1]{\darkredtext{Tab.~\ref*{#1}}}}
\newcommand{\appref}[1]{\hyperref[#1]{\darkredtext{App.~\ref*{#1}}}}

\newtheoremstyle{custom}
{1pt} %
{1pt} %
{\itshape} %
{} %
{\bfseries} %
{} %
{ } %
{\thmname{#1} \thmnumber{#2} \thmnote{(#3)} . } %

\theoremstyle{custom}

\newtheorem{innerdefinition}{Definition}
\newtheorem{innerproposition}{Proposition}
\newtheorem{innerassumption}{Assumption}
\newtheorem{innerremark}{Remark}
\newtheorem{innertheorem}{Theorem}
\newtheorem{innerhypothesis}{Hypothesis}
\newtheorem{innerconjecture}{Conjecture}
\newtheorem{innerlemma}{Lemma}
\newtheorem{innercorollary}{Corollary}
\newtheorem{innernotation}{Notation}
\newtheorem{innerclaim}{Claim}
\newtheorem{innerproblem}{Problem}

\newtheorem{innerobservation}{Observation}

\newmdenv[
  backgroundcolor=gray!10,
  linecolor=gray!100,
  linewidth=0.8pt,
  skipabove=2pt,
  skipbelow=2pt,
  innertopmargin=10pt,
  innerbottommargin=5pt,
  innerleftmargin=5pt,
  innerrightmargin=5pt,
]{definitionframe}

\newmdenv[
  backgroundcolor=blue!10,
  linecolor=blue!100,
  linewidth=0.8pt,
  skipabove=2pt,
  skipbelow=2pt,
  innertopmargin=10pt,
  innerbottommargin=5pt,
  innerleftmargin=5pt,
  innerrightmargin=5pt,
]{propositionframe}

\newmdenv[
  backgroundcolor=green!10,
  linecolor=green!100,
  linewidth=0.8pt,
  skipabove=2pt,
  skipbelow=2pt,
  innertopmargin=10pt,
  innerbottommargin=5pt,
  innerleftmargin=5pt,
  innerrightmargin=5pt,
]{assumptionframe}

\newmdenv[
  backgroundcolor=yellow!10,
  linecolor=yellow!100,
  linewidth=0.8pt,
  skipabove=2pt,
  skipbelow=2pt,
  innertopmargin=10pt,
  innerbottommargin=5pt,
  innerleftmargin=5pt,
  innerrightmargin=5pt,
]{remarkframe}

\newmdenv[
  backgroundcolor=red!10,
  linecolor=red!100,
  linewidth=0.8pt,
  skipabove=2pt,
  skipbelow=2pt,
  innertopmargin=10pt,
  innerbottommargin=5pt,
  innerleftmargin=5pt,
  innerrightmargin=5pt,
]{theoremframe}

\newmdenv[
  backgroundcolor=purple!10,
  linecolor=purple!100,
  linewidth=0.8pt,
  skipabove=2pt,
  skipbelow=2pt,
  innertopmargin=10pt,
  innerbottommargin=5pt,
  innerleftmargin=5pt,
  innerrightmargin=5pt,
]{hypothesisframe}

\newmdenv[
  backgroundcolor=orange!10,
  linecolor=orange!100,
  linewidth=0.8pt,
  skipabove=2pt,
  skipbelow=2pt,
  innertopmargin=10pt,
  innerbottommargin=5pt,
  innerleftmargin=5pt,
  innerrightmargin=5pt,
]{conjectureframe}

\newmdenv[
  backgroundcolor=cyan!10,
  linecolor=cyan!100,
  linewidth=0.8pt,
  skipabove=2pt,
  skipbelow=2pt,
  innertopmargin=10pt,
  innerbottommargin=5pt,
  innerleftmargin=5pt,
  innerrightmargin=5pt,
]{lemmaframe}

\newmdenv[
  backgroundcolor=magenta!10,
  linecolor=magenta!100,
  linewidth=0.8pt,
  skipabove=2pt,
  skipbelow=2pt,
  innertopmargin=10pt,
  innerbottommargin=5pt,
  innerleftmargin=5pt,
  innerrightmargin=5pt,
]{corollaryframe}

\newmdenv[
  backgroundcolor=pink!10,
  linecolor=pink!100,
  linewidth=0.8pt,
  skipabove=2pt,
  skipbelow=2pt,
  innertopmargin=10pt,
  innerbottommargin=5pt,
  innerleftmargin=5pt,
  innerrightmargin=5pt,
]{notationframe}

\newmdenv[
  backgroundcolor=violet!10,
  linecolor=violet!100,
  linewidth=0.8pt,
  skipabove=2pt,
  skipbelow=2pt,
  innertopmargin=10pt,
  innerbottommargin=5pt,
  innerleftmargin=5pt,
  innerrightmargin=5pt,
]{claimframe}

\newmdenv[
  backgroundcolor=salmon!10,
  linecolor=salmon!100,
  linewidth=0.8pt,
  skipabove=2pt,
  skipbelow=-2pt,
  innertopmargin=10pt,
  innerbottommargin=5pt,
  innerleftmargin=5pt,
  innerrightmargin=5pt,
]{problemframe}

\newmdenv[
  backgroundcolor=lavender!10,
  linecolor=lavender!100,
  linewidth=0.8pt,
  skipabove=2pt,
  skipbelow=2pt,
  innertopmargin=10pt,
  innerbottommargin=5pt,
  innerleftmargin=5pt,
  innerrightmargin=5pt,
]{observationframe}

\tcbset{
  thmbox/.style={
    enhanced,
    breakable,
    sharp corners,
    boxrule=0pt,
    leftrule=3pt,
    top=0pt,
    bottom=0pt,
    left=5pt,
    right=5pt,
    before skip=10pt,
    after skip=10pt,
  }
}

\newenvironment{proposition}
{\begin{tcolorbox}[thmbox, colback=blue!5!white, colframe=blue!75!black]\begin{innerproposition}}
      {\end{innerproposition}\end{tcolorbox}}

\newenvironment{theorem}
{\begin{tcolorbox}[thmbox, colback=red!5!white, colframe=red!75!black]\begin{innertheorem}}
      {\end{innertheorem}\end{tcolorbox}}

\newenvironment{corollary}
{\begin{tcolorbox}[thmbox, colback=magenta!5!white, colframe=magenta!75!black]\begin{innercorollary}}
      {\end{innercorollary}\end{tcolorbox}}

\newenvironment{problem}
{\begin{tcolorbox}[thmbox, colback=salmon!5!white, colframe=salmon!75!black]\begin{innerproblem}}
      {\end{innerproblem}\end{tcolorbox}}

\newcommand{\method}{\textsc{TBSM}\xspace} %
\newcommand{\sg}{\mathop{{\color{darkblue}\operatorname{sg}}}\nolimits}

\title{\fontsize{16.5}{20}\selectfont Three-Body Scattering for Generative Modeling}

\author{Peng Sun$^{1,2}$, Zhenglin Cheng$^{1,2}$, Deyuan Liu$^1$, Jun Xie$^{1,2}$, Xinyi Shang$^3$, Tao Lin$^{1}$ \\
$^1$Westlake University, $^2$Zhejiang University, $^3$University College London\\
\texttt{sp12138sp@gmail.com}\\
}

\iclrfinalcopy %
\begin{document}

\maketitle

\begin{abstract}
    Modern generative models typically rely on an adversarial critic, a prescribed noise-to-data path, or an autoregressive factorization.
    Instead, we show that a proper distributional energy can induce sample-level motion and provide direct regression supervision for a one-step generator.
    Three-Body Scattering Modeling (\method) for generation turns the energy distance into a constant-size per-projectile interaction: each projectile is attracted toward one real source and repelled from one independently generated source.
    Conditioned on the projectile and its condition, its expectation equals the $2$-Wasserstein gradient-flow velocity of $\frac12D_E^2(P_{\mtheta},Q)$.
    A batch of $B$ frozen-target events yields $O(B)$ sample-level losses, each using one reference for its condition instead of the minibatch-wide all-pairs field used by methods such as Drifting Models.
    Tracking this conditional expectation online can reduce field noise.
    Using scattering in frozen image features, \method trains one-step generators on ImageNet-256, achieving FID${}=2.23$ with pixel-space PixelDiT-XL and FID${}=1.63$ with latent-space DiT-XL at NFE${}=1$.
    We provide a design map relating diffusion-related supervision, Drift-like dynamics, and GAN-like objectives.
    These results establish tracked scattering as a route to high-dimensional one-step generation.
    Code: \url{https://github.com/sp12138/TBSM}.
\end{abstract}

\begin{figure*}[!b]
    \centering
    \includegraphics[width=\linewidth]{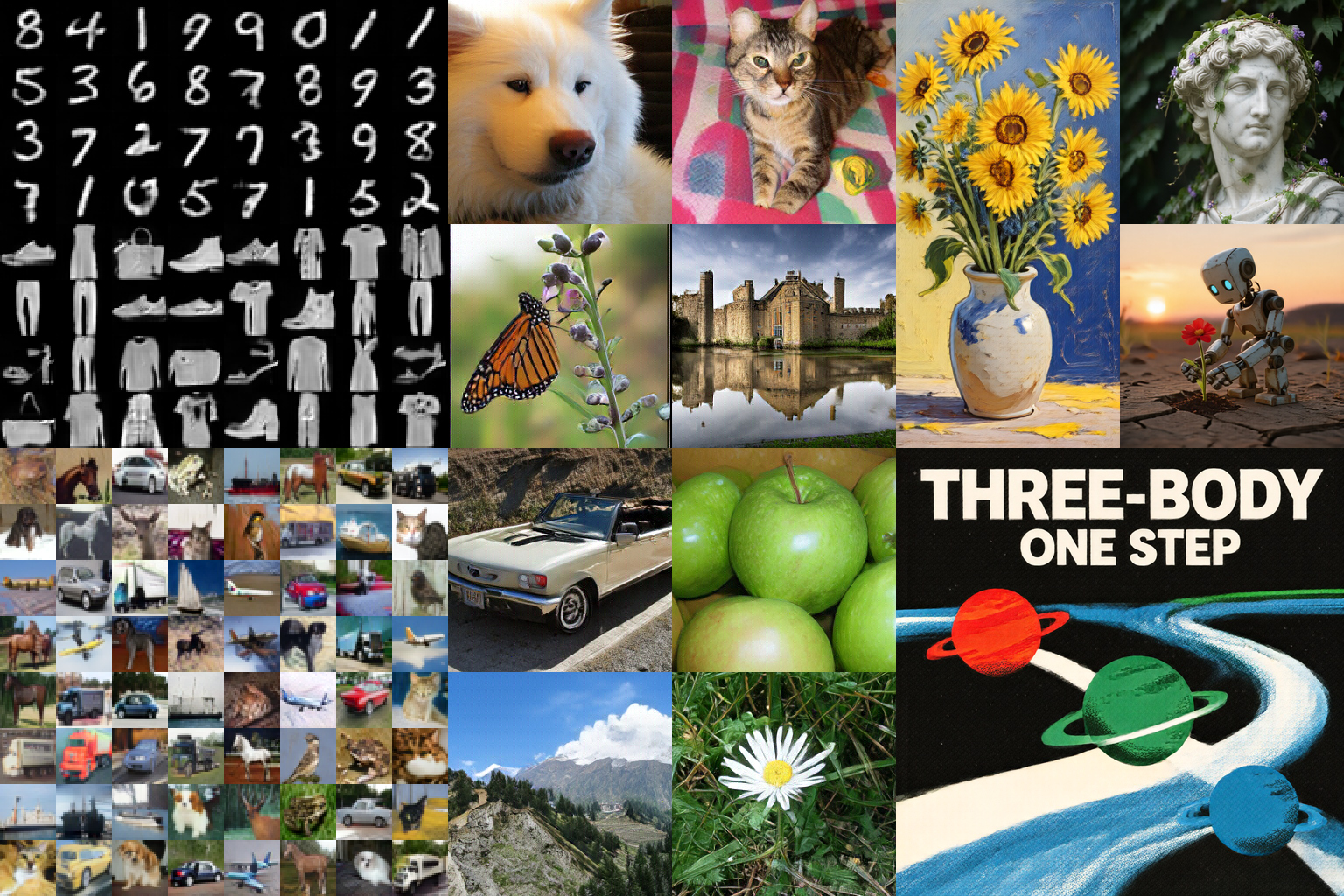}

    \caption{
        \textbf{Samples generated in one step by \method{}-trained models.}
        From left to right: MNIST, Fashion-MNIST, and CIFAR-10 (top to bottom); \textbf{\textit{pixel-space}} ImageNet-1K samples; and text-to-image samples generated at \textbf{NFE${}=1$ by a Qwen-Image-20B~\citep{wu2025qwenimagetechnicalreport} model fine-tuned using only our \method{} (see \appref{app:text-to-image})}.
    }
    \label{fig:scale-overview}
\end{figure*}

\section{Introduction}
\label{sec:intro}

Generative models rely on adversarial comparison~\citep{goodfellow2014generative,brock2018large}, path-indexed diffusion, score, or flow supervision~\citep{ho2020denoising,song2020score,lipman2022flow}, or autoregressive prediction~\citep{oord2016pixel,tian2024visual,sun2024autoregressive}.
Direct one-step distribution-matching methods usually estimate supervision from minibatch fields, statistics, or variational objectives~\citep{deng2026drifting,yang2026representation,caucheteux2026unifying}.

\begin{problem}
\textit{Can a proper distribution-matching objective yield constant-size per-projectile stochastic interactions for high-dimensional one-step generation without teacher queries?}
\end{problem}

We start from the exact population objective $\frac{1}{2}D_E^2(P_{\mtheta},Q)$, a proper discrepancy whose first variation yields particle-wise descent vectors~\citep{szekely2013energy,ambrosio2008gradient}.

Three-Body Scattering Modeling (\method) for generation treats a generated sample as a projectile attracted to one real source and repelled from an independent generated source; the latter supplies the model--model interaction absent from pure attraction.
Averaging this interaction recovers the $2$-Wasserstein descent velocity, while one triplet per projectile gives $O(B)$ detached-source targets for $B$ projectiles.
We call this supervision sample-level because each condition-specific loss requires only one paired real observation, naturally fitting text--image datasets with one image per caption rather than same-condition reference sets.
This projectile-level update is the classical stochastic gradient behind linear-time distance-kernel MMD~\citep{gretton2012kernel}, deployed with frozen targets and online tracking in high dimensions.

Following detached-target interfaces in direct distribution dynamics and representation matching~\citep{deng2026drifting,yang2026representation}, frozen-target regression makes the current-parameter gradient the generator-Jacobian pullback of the sampled vector up to loss scaling.
An online tracker approximates the conditional expectation of this cheap but noisy vector.
At the exact endpoint, oracle tracking preserves the population field and removes source-sampling variance; finite tracker error admits exact field- and update-space decompositions.
\figref{fig:method-overview} illustrates this construction.

\begin{figure*}[t]
    \centering
    \definecolor{tbsmParticleRed}{HTML}{CF6268}
    \definecolor{tbsmParticlePurple}{HTML}{9F69BF}
    \definecolor{tbsmParticleBlue}{HTML}{5695CA}
    \resizebox{\linewidth}{!}{\input{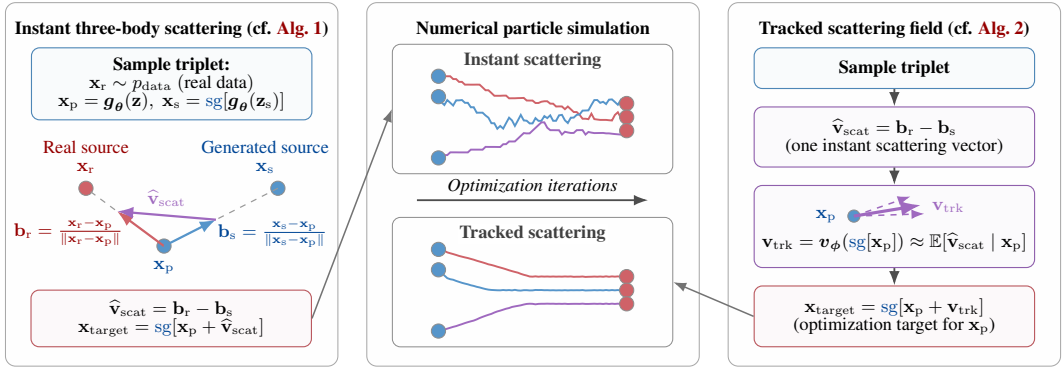}}
    \caption{
        \textbf{Overview of \method.}
        Left: a real--generated--generated event yields a noisy scattering vector.
        Center: instant scattering drives convergence toward the target along noisy trajectories; tracking smooths the motion.
        Right: at the fully tracked energy-distance endpoint, an online tracker estimates the conditional expectation of the instant scattering vector at each projectile and supervises the generator.
        Center-panel \textcolor{tbsmParticleRed}{red}, \textcolor{tbsmParticleBlue}{blue}, and \textcolor{tbsmParticlePurple}{purple} curves show generated-particle paths from \textcolor{tbsmParticleBlue}{blue initial points}; \textcolor{tbsmParticleRed}{red dots} mark targets; $\sg[\cdot]$ denotes stop gradient.
    }
    \label{fig:method-overview}
\end{figure*}

The exact energy-distance field anchors the theory; experiments retain the regression loop in frozen representation spaces with smoothed bearings and sometimes reweighted intra-source interaction.
The \method objective requires no path-model teacher queries, prescribed denoising target, adversarial discriminator, or ordered factorization.
Training and inference use no classifier-free guidance or guidance-scale sweep; \appref{app:positioning} contrasts pipelines that bake guidance into a training field or distillation target.
Primary ImageNet runs use multi-step checkpoints only for initialization (\appref{app:exp}).
\tabref{tab:paradigm-positioning} compares paradigms by sampling form, scaling evidence, and supervision granularity.
At NFE${}=1$, PixelDiT-XL (pixel) and DiT-XL/2 (latent) reach FID${}=2.23$ and $1.63$ on ImageNet-256; DiT-XL/4 reaches $1.92$ on ImageNet-512.
\figref{fig:scale-overview} adds small-scale and text-to-image samples.

Under an explicit slope condition, sufficiently small relative tracker excess risk preserves exact full-space energy-distance convergence.
Stronger parameter-space assumptions yield expected finite-generator convergence.
Finite-error bounds quantify deployed surrogates and smoothing, while representation-space guarantees concern only projected distributions.
These results do not establish convergence for general neural-network training.
Our contributions are:
\begin{enumerate}[label=(\alph*), nosep, leftmargin=16pt]
    \item \textbf{High-dimensional three-body scattering.} We turn energy-distance motion into constant-size frozen-target supervision for high-dimensional one-step generators with pixel or latent outputs using frozen image-representation fields.
    \item \textbf{Online tracked scattering.} We cast tracking as an online conditional-expectation projection and derive exact field- and update-space error decompositions, finite-error stationarity, and energy-distance convergence under explicit assumptions; matched-update corners favor tracked over instant scattering.
    \item \textbf{An interpretive design map.} We connect \method to Drift-like dynamics, a path-regularized GAN-like displacement analogue, and diffusion-related supervision via input and target changes.
\end{enumerate}

\begin{table*}[t]
    \centering
    \caption{
        \textbf{Positioning of major generative modeling paradigms.}
        We compare direct single-pass sampling, the breadth of empirical scaling evidence, and the granularity of generator supervision.
    }
    \scriptsize
    \setlength{\tabcolsep}{3pt}
    \renewcommand{\arraystretch}{1.08}
    \begin{tabular*}{\textwidth}{@{\extracolsep{\fill}}
        >{\RaggedRight\arraybackslash}p{0.32\textwidth}
        >{\RaggedRight\arraybackslash}p{0.20\textwidth}
        >{\RaggedRight\arraybackslash}p{0.20\textwidth}
        >{\RaggedRight\arraybackslash}p{0.20\textwidth}@{}}
        \toprule
        \textbf{Paradigm / representative work}
        & \textbf{Direct single-pass}\newline\textbf{sampling?}
        & \textbf{Scaling evidence?}
        & \textbf{Generator-supervision}\newline\textbf{granularity}
        \\
        \midrule
        GANs~\citep{goodfellow2014generative,brock2018large,sauer2022stylegan,kang2023scaling}
        & \statusyes{Yes}\newline direct single-pass generation
        & \statusother{Mixed}\newline several large-scale successes
        & \statusyes{Per-sample}\newline signal from a learned critic
        \\
        \midrule
        Diffusion, score, and flow models~\citep{ho2020denoising,song2020score,lipman2022flow,ma2024sit}
        & \statusno{No}\newline iterative path-based sampling
        & \statusyes{Strong}\newline extensive evidence at large scale
        & \statusyes{Per-sample}\newline denoising, score, or velocity target
        \\
        \midrule
        Autoregressive models~\citep{oord2016pixel,tian2024visual,sun2024autoregressive}
        & \statusno{No}\newline iterative structured decoding
        & \statusyes{Strong}\newline extensive evidence at large scale
        & \statusyes{Per-token}\newline predictive likelihood target
        \\
        \midrule
        One-step or few-step diffusion acceleration~\citep{song2023consistency,frans2024one,geng2025mean,yin2024one}
        & \statusother{Method-dependent}\newline one or a few generator evaluations
        & \statusother{Growing}\newline demonstrated on large image models
        & \statusother{Method-dependent}\newline teacher, consistency, or auxiliary target
        \\
        \midrule
        Direct distribution dynamics and distributional losses~\citep{deng2026drifting,yang2026representation,caucheteux2026unifying}
        & \statusyes{Yes}\newline direct single-pass generation
        & \statusother{Growing}\newline recent ImageNet-scale results
        & \statusother{Batch-estimated}\newline fields, statistics, or objectives
        \\
        \bottomrule
    \end{tabular*}
    \label{tab:paradigm-positioning}
\end{table*}

\section{Methodology}
\label{sec:methodology}

\subsection{Distributional Objective and Stochastic Optimization}
\label{sec:preliminaries}

\paragraph{Population objective.}
Let $Q(\cdot\mid\cc)$ be the real probability law on $\R^d$ and let $\mg_{\mtheta}(\zz,\cc)$ be a conditional generator with $\zz \sim p(\zz)$ and optional condition $\cc$.
For each condition, the generator induces the probability law $P_{\mtheta}(\cdot\mid\cc)=(\mg_{\mtheta}(\cdot,\cc))_\#p(\zz)$; in the unconditional case we write these laws as $P_{\mtheta}$ and $Q$.
We use $\xx$ for a generic spatial argument of a potential or vector field and reserve $\xx_{\mathrm{p}}$ for a sampled generated projectile in a stochastic three-body event.
The learning goal is condition-wise distribution matching:
\begin{equation}
    \cF(\mtheta)
    =
    \EEb{\cc\sim p_{\mathrm{data}}(\cc)}
    {
        \frac{1}{2}
        D_E^2
        \left(
        P_{\mtheta}(\cdot\mid\cc),
        Q(\cdot\mid\cc)
        \right)
    }.
    \label{eq:conditional-energy}
\end{equation}
The unconditional case is recovered by dropping $\cc$.
For finite paired datasets with one observation per condition, $Q(\cdot\mid\cc)$ is a population object rather than the literal empirical distribution for that condition.
We use the squared energy distance as the population discrepancy:
\begin{equation}
    \begin{aligned}
        D_E^2(P_{\mtheta},Q)
        =
         & \;2\EEb{\xx_{\mathrm{p}}\sim P_{\mtheta},\xx_{\mathrm{r}}\sim Q}
        {\norm{\xx_{\mathrm{p}}-\xx_{\mathrm{r}}}}
        \\
         & -
        \EEb{\xx_{\mathrm{p}},\xx_{\mathrm{s}}\sim P_{\mtheta}}
        {\norm{\xx_{\mathrm{p}}-\xx_{\mathrm{s}}}}
        -
        \EEb{\xx_{\mathrm{r}},\xx_{\mathrm{r}}'\sim Q}
        {\norm{\xx_{\mathrm{r}}-\xx_{\mathrm{r}}'}}.
    \end{aligned}
    \label{eq:energy-distance}
\end{equation}
Euclidean space is of strong negative type, so $D_E^2(P_{\mtheta},Q)\ge 0$ and equality holds exactly when $P_{\mtheta}=Q$ under finite first moments.
Equivalently, $\frac12D_E^2$ is a squared distance-kernel MMD; \appref{app:linear-mmd-relation} states its precise relation to the stochastic update below.

\paragraph{Population vector field.}
For a fixed condition, suppress $\cc$ and write $\cF(P_{\mtheta})=\frac{1}{2}D_E^2(P_{\mtheta},Q)$.
Ignoring the real-real constant, a representative of its first variation, defined up to an additive constant, is
\begin{equation}
    \frac{\delta \cF}{\delta P_{\mtheta}}(\xx)
    =
    \EEb{
        \substack{
            \xx_{\mathrm{r}}\sim Q,\;
            \xx_{\mathrm{s}}\sim P_{\mtheta}
        }
    }
    {
        \norm{\xx-\xx_{\mathrm{r}}}
        -
        \norm{\xx-\xx_{\mathrm{s}}}
    }.
    \label{eq:first-variation-potential}
\end{equation}
The coefficient of the generated-source term is one because the model-model term in \eqref{eq:energy-distance} depends on $P_{\mtheta}$ through both arguments; the symmetric variation cancels the factor $1/2$ in the definition of $\cF$.
The $2$-Wasserstein negative-gradient flow of a functional $\cF(P)$ has velocity $\vv(\xx)=-\nabla_{\xx}\frac{\delta \cF}{\delta P}(\xx)$.
Applying this identity to \eqref{eq:first-variation-potential} condition-wise gives the generated-particle scattering vector field
\begin{equation}
    \vv_{\mtheta}(\xx,\cc)
    =
    \EEb{\xx_{\mathrm{r}}\sim Q(\cdot\mid\cc)}
    {
        \frac{\xx_{\mathrm{r}}-\xx}
        {\norm{\xx_{\mathrm{r}}-\xx}}
    }
    -
    \EEb{\xx_{\mathrm{s}}\sim P_{\mtheta}(\cdot\mid\cc)}
    {
        \frac{\xx_{\mathrm{s}}-\xx}
        {\norm{\xx_{\mathrm{s}}-\xx}}
    }.
    \label{eq:population-field}
\end{equation}
These inter- and intra-source expected bearings have unit-norm integrands away from coincidence.
For estimator identities we set each bearing to zero at exact coincidence; code uses denominator smoothing, which is the exact gradient of a differentiable radial potential and converges pointwise to the bearing field as detailed in \appref{app:denominator-smoothing}.
The ideal flow follows $\dm\xx/\dm t=\vv_{\mtheta}(\xx,\cc)$, and its generator update pulls this field back through the generator Jacobian:
\begin{equation}
    \nabla_{\mtheta}\cF
    =
    -
    \EEb{\cc\sim p_{\mathrm{data}}(\cc),\,\zz\sim p(\zz)}
    {
        J_{\mtheta}\mg_{\mtheta}(\zz,\cc)^{\top}
        \vv_{\mtheta}(\mg_{\mtheta}(\zz,\cc),\cc)
    }.
    \label{eq:generator-gradient-field}
\end{equation}
\paragraph{Stochastic estimator.}
A full empirical evaluation of \eqref{eq:energy-distance} uses pairwise distances, whereas classical random-pair MMD estimators are linear in batch size~\citep{gretton2012kernel}.
We use one side of this stochastic gradient as a projectile-level field with one differentiable projectile, one real source, one detached generated source, and an explicit target that an auxiliary field can track; its $O(B)$ scaling refers to sample interactions, not per-sample network cost.
The exact relation to the symmetrized estimator is given in \appref{app:linear-mmd-relation}.

\begin{algorithm}[t]
    \caption{Three-Body Scattering Regression for Simple One-Step Generator Training}
    \label{alg:three-body-training}
    \begin{algorithmic}[1]
        \STATE Initialize generator parameters $\mtheta$.
        \REPEAT
        \STATE Sample real source-condition pairs $(\xx_{\mathrm{r}},\cc)$ from dataset.
        \STATE Sample noise $\zz,\zz_{\mathrm{s}}\sim p(\zz)$, and generate $\xx_{\mathrm{p}}=\mg_{\mtheta}(\zz,\cc)$ and $\xx_{\mathrm{s}}=\sg[\mg_{\mtheta}(\zz_{\mathrm{s}},\cc)]$.
        \STATE Compute inter-source bearing:
        $
            \bb_{\mathrm{r}} =
            (\xx_{\mathrm{r}}-\xx_{\mathrm{p}})
            /
            \norm{\xx_{\mathrm{r}}-\xx_{\mathrm{p}}}.
        $
        \STATE Compute intra-source bearing:
        $
            \bb_{\mathrm{s}} =
            (\xx_{\mathrm{s}}-\xx_{\mathrm{p}})
            /
            \norm{\xx_{\mathrm{s}}-\xx_{\mathrm{p}}}.
        $
        \STATE Update $\mtheta$ by descending
        $
            \norm{\xx_{\mathrm{p}}-\sg[\xx_{\mathrm{p}}+\bb_{\mathrm{r}}-\bb_{\mathrm{s}}]}_2^2.
        $
        \UNTIL{convergence}
    \end{algorithmic}
\end{algorithm}

\subsection{Three-Body Scattering Vector Field}
\label{sec:three-body-field}

\paragraph{Triplet estimator.}
We estimate \eqref{eq:population-field} with a constant-size signed interaction; \appref{app:minimal-three-body} makes precise why a real-only estimator cannot recover its model-dependent term.
For a condition $\cc$ and projectile $\xx_{\mathrm{p}}\sim P_{\mtheta}(\cdot\mid\cc)$, sample one real source $\xx_{\mathrm{r}}\sim Q(\cdot\mid\cc)$ and one independently generated source $\xx_{\mathrm{s}}\sim P_{\mtheta}(\cdot\mid\cc)$.
The corresponding sample-level scattering estimator is
\begin{equation}
    \widehat{\vv}_{\mathrm{scat}}(\xx_{\mathrm{p}};\xx_{\mathrm{r}},\xx_{\mathrm{s}})
    =
    \bb_{\mathrm{r}}-\bb_{\mathrm{s}},
    \qquad
    \bb_{\mathrm{r}}
    =
    \underbrace{
        \frac{\xx_{\mathrm{r}}-\xx_{\mathrm{p}}}
        {\norm{\xx_{\mathrm{r}}-\xx_{\mathrm{p}}}}
    }_{\text{inter-source bearing}},
    \qquad
    \bb_{\mathrm{s}}
    =
    \underbrace{
        \frac{\xx_{\mathrm{s}}-\xx_{\mathrm{p}}}
        {\norm{\xx_{\mathrm{s}}-\xx_{\mathrm{p}}}}
    }_{\text{intra-source bearing}},
    \label{eq:scattering-direction}
\end{equation}
Averaging over the two sources recovers the population scattering vector field:
\begin{equation}
    \EEb{\xx_{\mathrm{r}}\sim Q(\cdot\mid\cc),\,
    \xx_{\mathrm{s}}\sim P_{\mtheta}(\cdot\mid\cc)}
    {
    \widehat{\vv}_{\mathrm{scat}}(\xx_{\mathrm{p}};\xx_{\mathrm{r}},\xx_{\mathrm{s}})
    }
    =
    \vv_{\mtheta}(\xx_{\mathrm{p}},\cc).
    \label{eq:three-body-field-expectation}
\end{equation}
Equivalently, the same vector is the negative projectile gradient of the single-event energy $\norm{\xx_{\mathrm{p}}-\xx_{\mathrm{r}}}-\norm{\xx_{\mathrm{p}}-\xx_{\mathrm{s}}}$, whose source expectation is \eqref{eq:first-variation-potential} evaluated at the projectile.

\paragraph{Scattering interpretation.}
Here, three-body denotes a projectile's signed interaction with real and model sources, not Newtonian dynamics; \thmref{thm:three-body-equivalence} formalizes its objective-to-update identity.

\subsection{Scattering Target Regression}
\label{sec:scattering-regression}

\paragraph{Frozen target.}
For each loss event, we sample a projectile, real source, and independently generated source, displace the projectile by the sampled inter-minus-intra vector, and regress the generator output toward this detached target.
The one-step target is
\begin{equation}
    \xx_{\mathrm{target}}
    =
    \sg
    \left[
    \xx_{\mathrm{p}}
    + \widehat{\vv}_{\mathrm{scat}}(\sg[\xx_{\mathrm{p}}];\xx_{\mathrm{r}},\xx_{\mathrm{s}})
    \right],
    \qquad
    \xx_{\mathrm{p}}=\mg_{\mtheta}(\zz,\cc),
    \quad
    \xx_{\mathrm{s}}=\sg[\mg_{\mtheta}(\zz_{\mathrm{s}},\cc)].
    \label{eq:three-body-target}
\end{equation}
The generator objective averages this target regression over triplets:
\begin{equation}
    \cL_{\mathrm{gen}}(\mtheta)
    =
    \frac{1}{2}
    \EEb{
        (\xx_{\mathrm{r}},\cc)\sim p_{\mathrm{data}}(\xx,\cc),\,
        \zz,\zz_{\mathrm{s}}\sim p(\zz)
    }
    {
        \norm{
            \mg_{\mtheta}(\zz,\cc) - \xx_{\mathrm{target}}
        }_2^2
    }.
    \label{eq:target-loss}
\end{equation}
Detaching freezes the stochastic field estimate, so gradients pass only through the projectile and pull the fixed vector back through the generator Jacobian.
\paragraph{Formal scope.}
Fix $\cc$ and current parameters $\bar{\mtheta}$, abbreviate $Q=Q(\cdot\mid\cc)$, write
$P_{\mtheta}=(\mg_{\mtheta}(\cdot,\cc))_\#p(\zz)$ and
$\cF_{\cc}(\mtheta)=\frac{1}{2}D_E^2(P_{\mtheta},Q)$, and let
$\cL_{\mathrm{gen},\cc}(\mtheta;\bar{\mtheta})$ freeze all target-side quantities at $\bar{\mtheta}$.
For independent $\xx_{\mathrm r}\sim Q$ and $\zz,\zz_{\mathrm s}\sim p(\zz)$, set
$\xx_{\mathrm p}=\mg_{\bar{\mtheta}}(\zz,\cc)$,
$J_{\bar{\mtheta}}=\left.\nabla_{\mtheta}\mg_{\mtheta}(\zz,\cc)\right|_{\mtheta=\bar{\mtheta}}$,
$\xx_{\mathrm s}=\mg_{\bar{\mtheta}}(\zz_{\mathrm s},\cc)$, and
$\widehat{\vv}=\widehat{\vv}_{\mathrm{scat}}(\xx_{\mathrm p};\xx_{\mathrm r},\xx_{\mathrm s})$.
We assume exact bearings, finite first moments, zero collision probability, almost-sure generator differentiability, and valid differentiation under the expectation; \appref{app:theory} states the full scope.

\begin{theorem}[Energy-distance three-body local gradient equivalence]
    \label{thm:three-body-equivalence}
    Source independence gives
    $\Eb{\widehat{\vv}\mid\xx_{\mathrm p},\cc}=\vv_{\bar{\mtheta}}(\xx_{\mathrm p},\cc)$ and, for the generator pullback, $\Eb{\widehat{\vv}\mid\zz,J_{\bar{\mtheta}},\cc}=\vv_{\bar{\mtheta}}(\xx_{\mathrm p},\cc)$.
    Consequently, at the current parameters, exact-bearing frozen-target regression obeys the local identity
    \begin{equation}
        \left.
        \nabla_{\mtheta}\cL_{\mathrm{gen},\cc}(\mtheta;\bar{\mtheta})
        \right|_{\mtheta=\bar{\mtheta}}
        =
        \nabla_{\mtheta}\cF_{\cc}(\bar{\mtheta}).
        \label{eq:core-three-body-gradient-equivalence}
    \end{equation}
\end{theorem}
Moreover, $\norm{\widehat{\vv}}_2\leq2$ almost surely and $\Eb{\norm{\widehat{\vv}-\vv_{\bar{\mtheta}}(\xx_{\mathrm p},\cc)}_2^2\mid\xx_{\mathrm p},\cc}\leq2$, a dimension-free bound; \propref{prop:triplet-minimality-optimality} gives the exact variance and one-draw-class optimality.
Whenever the $2$-Wasserstein chain rule is valid, the exact population flow dissipates energy at rate $\frac{\dm}{\dm t}\frac{1}{2}D_E^2(P_t,Q)=-\int\norm{\vv_t(\xx)}_2^2P_t(\dm\xx)\leq0$.
One joint real--condition draw gives an unbiased gradient estimate for \eqref{eq:conditional-energy}, since the real--real term is $\mtheta$-independent; the objective vanishes exactly when conditional laws agree $p_{\mathrm{data}}(\cc)$-almost surely.

\paragraph{Gradient surrogate.}
Since $\xx_{\mathrm{target}}$ is detached, the per-projectile loss has output gradient $-\widehat{\vv}_{\mathrm{scat}}$.
Thus the regression is a local loss surrogate: with the $1/2$ convention, its expected current-parameter gradient equals $\nabla_{\mtheta}\cF$, while its value is not an estimator of the energy distance.
\algref{alg:three-body-training} omits that conventional factor because the optimizer learning rate absorbs the resulting gradient scale.
The per-sample pullback and the formal local-objective proof are given in \appref{app:gradient-surrogate}.

\subsection{Tracked Scattering Targets}
\label{sec:tracked-scattering}

\paragraph{Motivation.}
The estimator in \eqref{eq:scattering-direction} is cheap but noisy, using one source per distribution.
A \emph{scattering tracker} approximates this conditional expectation online, lowering target error when its error is below residual source-sampling variance.

\paragraph{Exact and weighted targets.}
Write $\widehat{\vv}_{\lambda}=\bb_{\mathrm{r}}-\lambda\bb_{\mathrm{s}}$, where the coupled design parameter $\lambda\in[0,1]$ controls both the intra-source coefficient and the tracker-query range defined below.
At the exact energy-distance endpoint, $\lambda=1$ and the tracker is queried at the projectile, where $\Eb{\widehat{\vv}_{\mathrm{scat}}\mid\xx_{\mathrm p},\cc}=\vv_{\mtheta}(\xx_{\mathrm p},\cc)$.
This population field is conservative. If a tracker constrained as $\mv_{\mphi}(\xx,\cc)=\nabla_{\xx}u_{\mphi}(\xx,\cc)$ matches it throughout a connected domain, then $u_{\mphi}$ equals the negative first variation there up to an additive constant.

More generally, before tracker conditioning, the exact-bearing source mean of $\widehat{\vv}_{\lambda}$ is the negative Wasserstein gradient field of
\begin{equation}
    \cF_{\lambda}(P;Q)
    =
    \EEb{\xx\sim P,\yy\sim Q}{\norm{\xx-\yy}}
    -
    \frac{\lambda}{2}
    \EEb{\xx,\xx'\sim P}{\norm{\xx-\xx'}}.
    \label{eq:lambda-weighted-functional}
\end{equation}
At $\lambda=1$, $\cF_{\lambda}$ equals $\cF(P)=\frac12D_E^2(P,Q)$ up to a $Q$-dependent constant, and $\widehat{\vv}_{\lambda}=\widehat{\vv}_{\mathrm{scat}}$.
For $\lambda<1$, $\cF_{\lambda}=\cF_1+\frac{1-\lambda}{2}\Eb{\norm{\xx-\xx'}}$, so this positive pairwise term favors lower generated-sample dispersion and generally shifts the optimum away from $Q$.

\paragraph{Tracker query and mixing.}
The tracker is queried at $\tilde{\xx}=\alpha\xx_{\mathrm{r}}+(1-\alpha)\sg[\xx_{\mathrm{p}}]$, with $\alpha\sim\mathcal{U}(0,1-\lambda)$ and point mass at $\alpha=0$ when $\lambda=1$.
This design choice, not a consequence of \eqref{eq:lambda-weighted-functional}, moves the query from the projectile at $\lambda=1$ to the fake-to-real corridor at $\lambda=0$.
For $\lambda<1$, $\tilde{\xx}$ generally does not determine the projectile, so its population regression target is a query-conditioned expectation generally distinct from the projectile-conditioned negative-gradient field induced by $\cF_{\lambda}$.
The tracker is trained by the regression objective $\cL_{\mathrm{trk}}(\mphi)=\Eb{\norm{\mv_{\mphi}(\tilde{\xx},\cc)-\sg[\widehat{\vv}_{\lambda}]}_2^2}$.
The tracked-supervision weight $\rho\in[0,1]$ then mixes the instantaneous and tracked vectors as $\widehat{\vv}_{\mathrm{mix}}=(1-\rho)\widehat{\vv}_{\lambda}+\rho\mv_{\mphi}(\tilde{\xx},\cc)$.
The generator target is still \eqref{eq:three-body-target}, with $\widehat{\vv}_{\mathrm{scat}}$ replaced by $\widehat{\vv}_{\mathrm{mix}}$.
For a fixed current generator and square-integrable tracker, at $\lambda=1$ define the conditional source variance
$s_{\mtheta}(\xx_{\mathrm p},\cc)=\operatorname{tr}\operatorname{Var}(\widehat{\vv}_{\mathrm{scat}}\mid\xx_{\mathrm p},\cc)$ and tracker error
$e_{\mphi}(\xx_{\mathrm p},\cc)=\norm{\mv_{\mphi}(\xx_{\mathrm p},\cc)-\vv_{\mtheta}(\xx_{\mathrm p},\cc)}_2^2$.

\Needspace{0.18\textheight}
\begin{corollary}[Finite-tracker error at the energy-distance endpoint]
    \label{cor:finite-tracker-error}
    At $\lambda=1$, the mixed scattering vector satisfies the exact decomposition
    \setlength{\belowdisplayskip}{-4pt}
    \begin{equation}
        \Eb{
        \norm{
        \widehat{\vv}_{\mathrm{mix}}
        -
        \vv_{\mtheta}(\xx_{\mathrm p},\cc)
        }_2^2
        \mid\xx_{\mathrm p},\cc
        }
        =
        (1-\rho)^2s_{\mtheta}(\xx_{\mathrm p},\cc)
        +
        \rho^2e_{\mphi}(\xx_{\mathrm p},\cc).
        \label{eq:finite-tracker-error-main}
    \end{equation}
\end{corollary}
Pointwise, the MSE-optimal mixture is $\rho^\star=s_{\mtheta}/(s_{\mtheta}+e_{\mphi})$ when the denominator is positive; at $\lambda=1$, full oracle tracking is the Rao--Blackwellization of the instant field estimator, eliminating source variance. Learned full tracking improves on the instant target exactly when tracker error is below that variance.
\thmref{thm:finite-tracker-decomposition} extends the identity to corridor-conditioned targets, and \corref{cor:parameter-tracker-decomposition} gives its generator-update-space counterpart.
At $\lambda=1$, tracker excess risk directly controls population energy dissipation and, under a slope condition, convergence; after Jacobian pullback, finite errors control stationarity.
For $\lambda\neq1$, the instant projectile-conditioned expectation follows \eqref{eq:lambda-weighted-functional}, whereas tracker regression targets a generally distinct query-conditioned expectation; even an oracle tracker need not preserve the $\cF_{\lambda}$ generator gradient after pullback, and finite tracker error adds approximation error.

\paragraph{Full tracked algorithm.}
In \algref{alg:tracked-three-body-training}, the three \textcolor{algorithmgray}{gray steps} after \textsc{repeat} are inherited from \algref{alg:three-body-training}.
For $\rho>0$, each sampled event supplies a detached mixed target for the generator and an instantaneous regression target for the tracker.
At $\rho=0$, the tracker query and update may be omitted; $\lambda=1$ then makes the generator update identical to \algref{alg:three-body-training}.
All reported experiments use this generator--tracker update structure with the representation interfaces in \secref{sec:practical-deployment}; \algref{alg:path-conditioned-gan-like-specialization} below is an expository displacement analogue only.

\begin{algorithm}[t]
    \caption{Full \method: Tracked Three-Body Scattering Regression}
    \label{alg:tracked-three-body-training}
    \begin{algorithmic}[1]
        \STATE Initialize generator parameters $\mtheta$ and tracker parameters $\mphi$.
        \STATE Choose tracked-supervision weight $\rho\in[0,1]$ and intra-source weight $\lambda\in[0,1]$.
        \REPEAT
        \STATE \textcolor{algorithmgray}{Sample real source-condition pairs $(\xx_{\mathrm{r}},\cc)$ from dataset.}
        \STATE \textcolor{algorithmgray}{Sample noise $\zz,\zz_{\mathrm{s}}\sim p(\zz)$, and generate $\xx_{\mathrm{p}}=\mg_{\mtheta}(\zz,\cc)$ and $\xx_{\mathrm{s}}=\sg[\mg_{\mtheta}(\zz_{\mathrm{s}},\cc)]$.}
        \STATE \textcolor{algorithmgray}{Compute bearings $\bb_{\mathrm{r}}=(\xx_{\mathrm{r}}-\xx_{\mathrm{p}})/\norm{\xx_{\mathrm{r}}-\xx_{\mathrm{p}}}$ and $\bb_{\mathrm{s}}=(\xx_{\mathrm{s}}-\xx_{\mathrm{p}})/\norm{\xx_{\mathrm{s}}-\xx_{\mathrm{p}}}$.}
        \STATE Form $\lambda$-weighted scattering vector $\widehat{\vv}_{\lambda}=\bb_{\mathrm{r}}-\lambda\bb_{\mathrm{s}}$.
        \STATE Sample $\alpha\sim\mathcal{U}(0,1-\lambda)$ and set $\tilde{\xx}=\alpha\xx_{\mathrm{r}}+(1-\alpha)\sg[\xx_{\mathrm{p}}]$.
        \STATE Query tracked vector $\vv_{\mathrm{trk}}=\mv_{\mphi}(\tilde{\xx},\cc)$.
        \STATE Form mixed scattering vector $\widehat{\vv}_{\mathrm{mix}}=(1-\rho)\widehat{\vv}_{\lambda}+\rho\vv_{\mathrm{trk}}$.
        \STATE Update $\mtheta$ by descending $\norm{\xx_{\mathrm{p}}-\sg[\xx_{\mathrm{p}}+\widehat{\vv}_{\mathrm{mix}}]}_2^2$.
        \STATE Update $\mphi$ by descending $\norm{\vv_{\mathrm{trk}}-\sg[\widehat{\vv}_{\lambda}]}_2^2$.
        \UNTIL{convergence}
    \end{algorithmic}
\end{algorithm}
\FloatBarrier

\begin{algorithm}[t]
    \caption{GAN-like Displacement Analogue of \method{} at $\rho=1,\lambda=0$}
    \label{alg:path-conditioned-gan-like-specialization}
    \begin{algorithmic}[1]
        \STATE Initialize generator parameters $\mtheta$ and tracker parameters $\mphi$.
        \REPEAT
        \STATE Sample real source-condition pairs $(\xx_{\mathrm{r}},\cc)$ from dataset.
        \STATE Sample noise $\zz\sim p(\zz)$ and generate projectile $\xx_{\mathrm{p}}=\mg_{\mtheta}(\zz,\cc)$.
        \STATE Sample $\alpha\sim\mathcal U(0,1)$ and set
        $
            \tilde{\xx}
            =
            \alpha\xx_{\mathrm{r}}
            +
            (1-\alpha)\sg[\xx_{\mathrm{p}}].
        $
        \STATE Query tracker field
        $
            \vv_{\mathrm{trk}}
            =
            \mv_{\mphi}(\tilde{\xx},\cc).
        $
        \STATE Update $\mtheta$ by descending
        $
            \norm{
                \xx_{\mathrm{p}}
                -
                \sg[
                    \xx_{\mathrm{p}}
                    +
                    \vv_{\mathrm{trk}}
                ]
            }_2^2.
        $
        \STATE Update $\mphi$ by descending
        $
            \norm{
                \vv_{\mathrm{trk}}
                -
                \sg
                [
                    \xx_{\mathrm{r}}-\sg[\xx_{\mathrm{p}}]
                ]
            }_2^2.
        $
        \UNTIL{convergence}
    \end{algorithmic}
\end{algorithm}

\subsection{Displacement Analogue of the GAN-Like State}
\label{sec:path-conditioned-fallback}

\paragraph{Fake-to-real state at $\rho=1,\lambda=0$.}
At $\rho=1,\lambda=0$, \algref{alg:tracked-three-body-training} uses the unit-bearing target $\bb_{\mathrm r}$ and queries the learned tracker along the fake-to-real corridor.
\algref{alg:path-conditioned-gan-like-specialization} preserves this corridor but replaces the unit bearing with the displacement $\xx_{\mathrm r}-\xx_{\mathrm p}$, so it is not a parameter-only specialization.
The resulting tracker learns both orientation and magnitude without a generated-source branch, while its detached output supervises the generator; feature normalization and the generator learning rate set the effective update scale.

\paragraph{Line-integral critic interpretation.}
The GAN-like connection becomes precise when the tracker is conservative,
$\mv_{\mphi}(\xx,\cc)=\nabla_{\xx}u_{\mphi}(\xx,\cc)$.
Let $\Delta=\xx_{\mathrm r}-\xx_{\mathrm p}$ and
$\xx_{\alpha}=(1-\alpha)\xx_{\mathrm p}+\alpha\xx_{\mathrm r}$.
The line-integral identity along $\xx_{\alpha}$ implies that, for a fixed generator,
\begin{equation}
    \cL_{\mathrm{trk}}^{\mathrm{disp}}(\mphi)
    =
    \Eb{\norm{\nabla_{\xx}u_{\mphi}(\xx_{\alpha},\cc)}_2^2}
    -2\left(
    \Eb{u_{\mphi}(\xx_{\mathrm r},\cc)}
    -
    \Eb{u_{\mphi}(\xx_{\mathrm p},\cc)}
    \right)
    + C,
    \label{eq:gan-like-critic-objective}
\end{equation}
where $C=\Eb{\norm{\Delta}_2^2}$ is independent of $\mphi$ during the tracker update.
Thus, for this conservative displacement analogue, tracker regression maximizes the real--generated potential gap minus one half of the corridor-averaged squared gradient norm; \appref{app:gan-like-line-integral} derives the identity and its conservative-field projection.
This is still not identical to conventional GAN minimax training: the generator queries the field at a random corridor point rather than necessarily at $\xx_{\mathrm p}$, and a general vector-valued tracker has no scalar-potential line-integral identity.

\subsection{Practical Deployment Guidance}
\label{sec:practical-deployment}

\paragraph{Exact endpoint and deployed surrogates.}
At $\lambda=1$, exact and smoothed source means are the negative Wasserstein-gradient fields of the proper objectives $\cF$ and $\cF_\varepsilon$ (\appref{app:denominator-smoothing}).
Tracker identities apply to both; tracked convergence additionally requires the corresponding chain rule, slope, and relative-error conditions.
For $\lambda<1$, \propref{prop:near-endpoint-reweighting} gives objective consistency as $\lambda\uparrow1$, while \propref{prop:practical-mixed-field-stability} combines reweighting, corridor, and tracker errors; under a linear endpoint slope, these errors yield an explicit convergence neighborhood.

\paragraph{Representation-space scattering.}
We use raw-pixel scattering for low-dimensional data and small images and scattering through frozen encoders for natural images.
For a fixed representation map $\mf$, \algref{alg:tracked-three-body-training} is applied to $\mf(\xx_{\mathrm p})$, $\mf(\xx_{\mathrm r})$, and $\mf(\xx_{\mathrm s})$; projectile gradients backpropagate through $\mf$ into the generator, while source features remain detached.
Pixel outputs enter $\mf$ directly, whereas latent outputs are first decoded, so this projects the loss computation rather than changing the generator output domain.
At $\lambda=1$, the exact and denominator-smoothed objectives are proper only for the projected distributions; at $\lambda<1$, the weighted functional is a projected surrogate.
Neither case guarantees full image matching without sufficient representation informativeness.

\paragraph{Tracker parameterization.}
The conservative tracker in \secref{sec:tracked-scattering} can be implemented as the input gradient of a scalar potential network, but higher-order autograd is expensive at image scale.
A lightweight vector-valued network with the representation dimension avoids higher-order autograd but drops exact conservativity.
See \appref{app:exp} for architecture details.

\paragraph{Multiple representation fields.}
Each fixed representation map defines its own feature-space field, query, and regression loss; vectors from incompatible spaces are never added.
At $\lambda=1$, exact or smoothed, the expected current-parameter gradient of the summed instant or oracle-tracked losses equals that of the corresponding sum of proper projected discrepancies; finite-tracker errors combine only after Jacobian pullback.
Zero summed discrepancy matches every projected marginal, not their joint or the full image law unless the representations are jointly measure-determining.

\begin{figure*}[!t]
    \centering
    \input{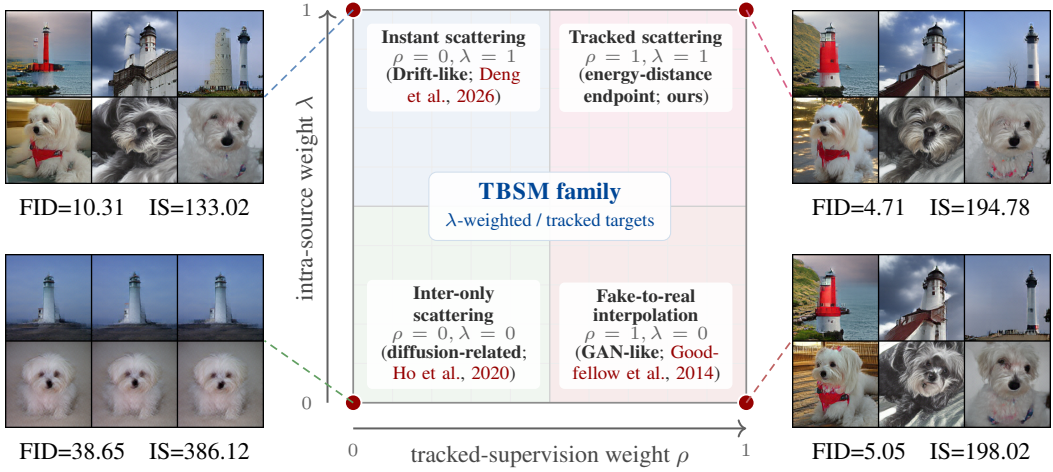}
    \caption{
        \textbf{One-step generated samples across the $(\rho,\lambda)$ design map.}
        See \tabref{tab:imagenet-training-config} for optimization settings.
    }
    \label{fig:lambda-rho-plane}
\end{figure*}

\section{The Generative Design Map}
\label{sec:lambda-rho-plane}

The $(\rho,\lambda)$ map organizes four reference configurations and interior settings within \algref{alg:tracked-three-body-training}.

\subsection{Reference Configurations and Matched Generator Updates}
\label{sec:corner-cases}

Here $\rho$ mixes the instantaneous stochastic vector and learned tracker field, while $\lambda$ jointly sets the intra-source coefficient in $\widehat{\vv}_{\lambda}=\bb_{\mathrm{r}}-\lambda\bb_{\mathrm{s}}$ and the tracker-query range $\alpha\in[0,1-\lambda]$.
Because $\lambda$ couples two changes, this is neither a factorized ablation nor a claim of exact equivalence to neighboring method families.

\paragraph{Instant scattering / Drift-like.}
At $\rho=0,\lambda=1$, the generator uses the instant energy-distance scattering update derived above; this is closest to Drift-like particle dynamics, but with a constant-size three-body estimator rather than a batch-level pairwise field.

\paragraph{Tracked scattering / energy-distance endpoint.}
At $\rho=1,\lambda=1$, the tracker is trained to estimate the same inter-minus-intra field; \corref{cor:finite-tracker-error} gives its exact error decomposition.

\paragraph{Inter-only scattering and the denoising connection.}
At $\rho=0,\lambda=0$, the one-step update reduces to instantaneous attraction toward an independently sampled real source and is not itself a diffusion objective.
An exact denoising-style reduction additionally requires paired $(\xx_t,\xx_0)$ across noise levels, conditioning on $t$, and a clean-sample, noise, or velocity target rather than a unit bearing; pure-noise input alone is insufficient.

\paragraph{Fake-to-real scattering and the GAN connection.}
At $\rho=1,\lambda=0$, \algref{alg:tracked-three-body-training} follows a learned field regressed from unit-bearing targets and queried along the fake-to-real corridor.
The displacement analogue in \algref{alg:path-conditioned-gan-like-specialization} exposes the potential-tracker interpretation in \eqref{eq:gan-like-critic-objective}, which does not transfer unchanged to the unit-bearing target.

\paragraph{Empirical comparison.}
At matched generator updates, $\rho=1,\lambda=1$ gives the lowest FID in \figref{fig:lambda-rho-plane}, consistent with retaining attraction and self-repulsion while replacing source-sampling noise with a learned field.
The $\rho=0,\lambda=0$ corner has the worst FID but highest IS, a pattern compatible with the contraction pressure induced by $\cF_{\lambda}$ when $\lambda<1$.
The comparison is not factorized: $\rho=0$ leaves the $\lambda$-specific field instantaneous, whereas changing $\lambda$ also changes the tracker corridor.
Interior $\rho$ trades instant variance against tracker error, while $\lambda$ near one retains self-repulsion; this motivates the near-endpoint benchmark settings in \secref{sec:image-exp}, without making the $\lambda<1$ objective proper.

\subsection{Connections to Related Distribution-Matching Objectives}
\label{sec:representation-objective-connections}

The following relations require restrictions or input/target modifications beyond choosing $(\rho,\lambda)$ and are interpretive rather than additional experimental algorithms.

\paragraph{A condition-only tracker and FD-loss.}
In the displacement analogue of \algref{alg:path-conditioned-gan-like-specialization}, let $\mu_{\mathrm r}(\cc)$ and $\mu_{\mtheta}(\cc)$ be the conditional real and generated feature means.
Restricting the representation-space tracker to $\cc$ alone gives $\mv^\star(\cc)=\mu_{\mathrm r}(\cc)-\mu_{\mtheta}(\cc)$ and generator descent on $\frac12\norm{\mu_{\mtheta}(\cc)-\mu_{\mathrm r}(\cc)}_2^2$.
This is an online, first-moment-only analogue of representation FD-loss~\citep{yang2026representation}, whose full Fr\'echet objective also matches feature covariances.

\paragraph{A displacement tracker and DMD.}
The displacement analogue in \algref{alg:path-conditioned-gan-like-specialization} resembles Distribution Matching Distillation (DMD)~\citep{yin2024one} through an auxiliary model following the generated distribution and supplying a detached generator direction.
DMD approximates a reverse-KL gradient from a fixed diffusion-teacher score and an online generated-distribution score at noise-perturbed generated samples; the analogue instead regresses fake-to-real displacements along straight corridors between independent endpoints, without teacher queries, score estimation, or a noise-time schedule.

\paragraph{Noisy-input inter-only scattering and perceptual supervision.}
If the inter-only corner receives a time-indexed noisy input $\xx_t$, its paired clean sample $\xx_0$, and condition $t$, its template approaches Perceptual Flow Matching (PFM)~\citep{zhao2026perceptual} and PixelGen~\citep{ma2026pixelgen}: each compares a clean prediction with $\xx_0$ through frozen perceptual features.
The objectives remain distinct because inter-only \method{} uses a normalized feature bearing and frozen-target regression, whereas PFM and PixelGen retain time-indexed flow/diffusion training with their perceptual losses.

\begin{table}[!t]
    \centering
    \caption{
        \textbf{Representation-field ablation on ImageNet-256.}
        Runtime is per step; see \tabref{tab:imagenet-training-config} for training settings.
    }
    \scriptsize
    \setlength{\tabcolsep}{4pt}
    \renewcommand{\arraystretch}{1.08}
    \begin{tabular*}{\textwidth}{@{\extracolsep{\fill}}lccccccc@{}}
        \toprule
        \textbf{Metric}
        & \textbf{ResNet-18}
        & \textbf{SigLIP2-B}
        & \textbf{MAE-B}
        & \shortstack{\textbf{ResNet-18}\\\textbf{+ SigLIP2-B}}
        & \shortstack{\textbf{SigLIP2-B}\\\textbf{+ MAE-B}}
        & \shortstack{\textbf{ResNet-18}\\\textbf{+ MAE-B}}
        & \shortstack{\textbf{ResNet-18}\\\textbf{+ MAE-B}\\\textbf{+ SigLIP2-B}}
        \\
        \midrule
        FID $\downarrow$ & 13.47 & 21.87 & 16.18 & 9.12 & 9.36 & 11.45 & \textbf{8.29} \\
        FDr$^6$ $\downarrow$   & 36.49 & 39.01 & 27.74 & 27.49 & \textbf{21.80} & 27.50 & 22.25 \\
        IS $\uparrow$          & 104.58 & 131.49 & 93.25 & 149.91 & \textbf{161.36} & 113.14 & 152.23 \\
        Avg. sec/step $\downarrow$ & \textbf{0.1389} & 0.1736 & 0.1583 & 0.1982 & 0.2184 & 0.1832 & 0.2411 \\
        \bottomrule
    \end{tabular*}
    \label{tab:representation-ablation}
\end{table}

\section{Experiments}
\label{sec:exp}

We evaluate representation fields and one-step ImageNet benchmarks; \figref{fig:lambda-rho-plane} compares design-map corners at matched generator updates, and \appref{app:additional-image-results} reports conversion and refinement studies.

\subsection{Experimental Setup}
\label{sec:expset}

\paragraph{Datasets.}
Our quantitative evaluation uses ImageNet-1K~\citep{deng2009imagenet} at $256\times256$ (ImageNet-256); \appref{app:additional-image-results} adds a $512\times512$ conversion study (ImageNet-512).
MNIST, Fashion-MNIST, and CIFAR-10 are used only for qualitative observations in \figref{fig:scale-overview}.

\paragraph{Network architectures.}
On ImageNet, we evaluate JiT~\citep{li2025back}, DiT~\citep{peebles2023scalable}, and PixelDiT~\citep{yu2025pixeldit} backbones.
JiT and PixelDiT operate in pixel space, while DiT operates on pretrained SD-VAE~\citep{rombach2022high} latents, following standard high-resolution latent-generation practice~\citep{dhariwal2021diffusion,karras2024analyzing,song2023consistency}.
For these qualitative observations, we use compact U-Net generators on all three small-scale datasets.
Tracker architectures and representation interfaces are detailed in \appref{app:exp}.

\paragraph{Implementation details.}
All models are implemented in PyTorch~\citep{paszke2019pytorch} and optimized with AdamW~\citep{loshchilov2017decoupled}.
To reduce training compute, all ImageNet-1K \method runs except the FD-loss refinement in \figref{fig:fdloss-refinement} initialize the generator from a pretrained multi-step model trained with diffusion or flow matching; see the \emph{Generator initialization} paragraph in \appref{app:exp}.
The instant baseline uses $\rho=0,\lambda=1$, with the unused tracker update omitted.
All reported \method{} ImageNet metrics use $50{,}000$ generated samples and training-set reference statistics.
We report FID~\citep{heusel2017gans}, FDr$^6$~\citep{yang2026representation} (the mean generated-to-training FD normalized by validation-to-training FD over six feature spaces), and Inception Score (IS).
NFE (number of function evaluations) counts the generator forward passes required to produce one sample.
The corner comparison matches the generator-update budget, but tracked variants additionally optimize the tracker, so total training compute is not matched.

\subsection{Representation-Field Ablation}
\label{sec:ablations}

Following FD-loss~\citep{yang2026representation}, we compare scattering vectors computed in single and multiple frozen feature spaces.
The encoders include a supervised ResNet-18~\citep{he2016deep}, SigLIP2-B~\citep{tschannen2025siglip2}, and MAE-B~\citep{he2022masked}.
For each encoder, we EMA-whiten and channelwise $\ell_2$-normalize its final-layer map, then use a separate tracker and equally weighted loss (\appref{app:exp}).
\tabref{tab:representation-ablation} is the only ImageNet-1K experiment that varies this set; every other ImageNet-1K experiment sums all three losses.

Within this matched-budget ablation, the three-encoder field gives the best FID, whereas SigLIP2-B + MAE-B gives the best FDr$^6$ and IS at lower per-step cost.
All combined fields containing SigLIP2-B have better reported values than every single encoder across the three quality metrics, while adding the third encoder trades additional computation for the lowest FID.

\begin{table}[!t]
    \centering
    \caption{
        \textbf{ImageNet benchmark comparison.}
        Space denotes the generator output domain, and Params counts inference-time generator parameters.
        \method{} rows show their $(\rho,\lambda)$ settings; $\dagger$ denotes longer training.
    }
    \scriptsize
    \setlength{\tabcolsep}{4pt}
    \renewcommand{\arraystretch}{1.08}
    \begin{tabular*}{\textwidth}{@{\extracolsep{\fill}}p{0.25\textwidth}p{0.08\textwidth}p{0.19\textwidth}p{0.06\textwidth}p{0.08\textwidth}ccc@{}}
        \toprule
        \textbf{Method}
        & \textbf{Space}
        & \textbf{Backbone}
        & \textbf{NFE}
        & \textbf{Params}
        & \textbf{FID $\downarrow$}
        & \textbf{FDr$^6$ $\downarrow$}
        & \textbf{IS $\uparrow$}
        \\
        \midrule
        \multicolumn{8}{c}{\textbf{Reference statistics}} \\
        \midrule
        Validation set images
        & pixel
        & -
        & -
        & -
        & 1.68
        & 1.00
        & 232.2
        \\
        \midrule
        \multicolumn{8}{c}{\textbf{Iterative diffusion and flow baselines}} \\
        \midrule
        ADM~\citep{dhariwal2021diffusion}
        & pixel
        & U-Net cascade
        & $250{+}250$
        & 608M
        & 3.94
        & --
        & 215.8
        \\
        VDM++~\citep{kingma2023understanding}
        & pixel
        & U-ViT-L/2
        & $512{\times}2$
        & 2B
        & 2.12
        & --
        & 267.7
        \\
        DiT~\citep{peebles2023scalable}
        & latent
        & DiT-XL/2
        & $250{\times}2$
        & 675M
        & 2.27
        & --
        & 278.2
        \\
        JiT~\citep{li2025back}
        & pixel
        & JiT-H/16
        & 200
        & 953M
        & 1.97
        & 7.66
        & 296.0
        \\
        PixelDiT~\citep{yu2025pixeldit}
        & pixel
        & PixelDiT-XL/16
        & $100{\times}2$
        & 797M
        & 1.61
        & --
        & 292.7
        \\
        \midrule
        \multicolumn{8}{c}{\textbf{Autoregressive visual generators}} \\
        \midrule
        VAR~\citep{tian2024visual}
        & discrete
        & VAR-d30
        & $10{\times}2$
        & 2B
        & 1.97
        & 6.70
        & 304.6
        \\
        BAR~\citep{yu2026bar}
        & discrete
        & BAR-L
        & 2048
        & 1.1B
        & 1.01
        & 3.57
        & 281.9
        \\
        MAR~\citep{li2024autoregressive}
        & latent
        & MAR-H
        & $51{,}200$
        & 942M
        & 1.56
        & 5.61
        & 299.5
        \\
        \midrule
        \multicolumn{8}{c}{\textbf{GAN-based one-step generators}} \\
        \midrule
        BigGAN~\citep{brock2018large}
        & pixel
        & BigGAN-deep
        & 1
        & --
        & 6.95
        & --
        & 171.4
        \\
        StyleGAN~\citep{sauer2022stylegan}
        & pixel
        & StyleGAN-XL
        & 1
        & --
        & 2.30
        & --
        & 265.1
        \\
        \midrule
        \multicolumn{8}{c}{\textbf{One-step and few-step diffusion acceleration}} \\
        \midrule
        iCT~\citep{song2023improved}
        & latent
        & DiT-XL/2
        & 1
        & 675M
        & 34.24
        & --
        & --
        \\
        Shortcut~\citep{frans2024one}
        & latent
        & DiT-XL/2
        & 1
        & 676M
        & 10.60
        & --
        & --
        \\
        MeanFlow~\citep{geng2025mean}
        & latent
        & DiT-XL/2
        & 1
        & 676M
        & 3.43
        & --
        & --
        \\
        pMF~\citep{lu2026pixelmeanflow}
        & pixel
        & pMF-H/16
        & 1
        & 935M
        & 2.29
        & 6.87
        & 267.2
        \\
        \midrule
        \multicolumn{8}{c}{\textbf{Direct distribution-dynamics and representation-distance methods}} \\
        \midrule
        Drift~\citep{deng2026drifting}
        & pixel
        & DiT-L/16
        & 1
        & 465M
        & 1.43
        & 10.51
        & 305.8
        \\
        FD-loss~\citep{yang2026representation}
        & pixel
        & JiT-H/16
        & 1
        & 953M
        & 0.75
        & 2.65
        & 313.0
        \\
        \midrule
        \multicolumn{8}{c}{\textbf{\method{} variants}} \\
        \midrule
        \rowcolor{gray!10}
        \method{} $(\rho{=}1.0,\lambda{=}1.0)$
        & pixel
        & JiT-B/16
        & 1
        & 131M
        & 3.09
        & 10.73
        & 208.6
        \\
        \rowcolor{gray!10}
        \method{} $(\rho{=}0.9,\lambda{=}1.0)$
        & pixel
        & JiT-B/16
        & 1
        & 131M
        & 3.35
        & 11.60
        & 205.3
        \\
        \rowcolor{gray!10}
        \method{} $(\rho{=}0.9,\lambda{=}0.9)$
        & pixel
        & JiT-B/16
        & 1
        & 131M
        & 2.92
        & 13.21
        & 228.7
        \\
        \rowcolor{gray!10}
        \method{}$^\dagger$ $(\rho{=}0.9,\lambda{=}0.9)$
        & pixel
        & JiT-B/16
        & 1
        & 131M
        & 2.69
        & 13.21
        & 247.4
        \\
        \rowcolor{gray!10}
        \method{} $(\rho{=}0.9,\lambda{=}1.0)$
        & pixel
        & PixelDiT-XL/16
        & 1
        & 797M
        & 2.23
        & 7.15
        & 210.0
        \\
        \rowcolor{gray!10}
        \method{} $(\rho{=}0.9,\lambda{=}0.9)$
        & pixel
        & PixelDiT-XL/16
        & 1
        & 797M
        & 2.23
        & 7.59
        & 230.9
        \\
        \rowcolor{gray!10}
        \method{} $(\rho{=}0.9,\lambda{=}0.9)$
        & latent
        & DiT-XL/2
        & 1
        & 675M
        & 1.63
        & 6.76
        & 243.0
        \\
        \bottomrule
    \end{tabular*}
    \label{tab:imagenet-benchmark}
\end{table}

\subsection{Benchmark Results}
\label{sec:image-exp}

\tabref{tab:imagenet-benchmark} shows that \method{} is competitive at NFE${}=1$ in both output spaces, reaching FID${}=2.23$ with pixel-space PixelDiT-XL and FID${}=1.63$ with latent-space DiT-XL/2; the JiT-B results further demonstrate compatibility with a substantially smaller pixel backbone.
At fixed $\rho=0.9$, reducing $\lambda$ from $1.0$ to $0.9$ produces an empirical tuning effect reminiscent of classifier-free guidance (CFG)~\citep{ho2022classifier}: on JiT-B, FID decreases from $3.35$ to $2.92$ while IS rises from $205.3$ to $228.7$; on PixelDiT-XL, FID rounds to $2.23$ in both settings while IS rises from $210.0$ to $230.9$.
This is an empirical analogy only: $\lambda$ weights intra-source scattering rather than implementing CFG.
Using the high-resolution curriculum in \appref{app:additional-image-results}, one-step DiT-XL/4 reaches FID${}=1.92$ and PixelDiT-XL reaches FID${}=3.84$ on ImageNet-512.
FD-loss~\citep{yang2026representation} has the strongest table scores; in a separate JiT-B case study, the random-sample grid after \method{} training in \figref{fig:fdloss-refinement} displays fewer block-like and grid-like artifacts despite worse metrics, illustrating a possible metric--artifact mismatch.
ImageNet optimizer, batch-size, and training-step settings appear in \appref{app:exp}.
The table is not exhaustive: W-Flow and Representation Distribution Matching are discussed in \appref{app:positioning} rather than included as numerical baselines because their supervision is built from minibatch-level distribution estimates, whereas this study isolates constant-size sample-level targets; we claim no numerical superiority over them~\citep{han2026wflow,feng2026representation}.

\section{Discussion}
\label{sec:discussion}

At its energy-distance endpoint, \method turns a proper distributional energy into sample-level motion.
The framework is close in spirit to particle systems and kernel discrepancies, but differs from classical sample optimization because the moving particles are outputs of a shared neural generator.
The main practical finding is that this constant-size interaction remains effective for high-dimensional image generation.
The resulting family also provides a common language for relating explicit distributional motion to Drift-like dynamics, GAN-like learned fields, and representation-space distribution matching.

\section{Limitations}
\label{sec:limitations}

First, population-flow convergence assumes the exact $\lambda=1$ field, chain-rule regularity, a slope inequality, and relative tracking; finite-generator convergence instead assumes a smooth exact objective, realizability, gradient dominance, relative bias, bounded variance, and diminishing steps.
Neither establishes convergence for general neural generators trained by SGD.
Second, vector estimates may be noisy near coincident particles; at $\lambda=1$, fixed denominator smoothing stabilizes them and retains a proper smoothed objective, but changes its geometry.
Third, the efficiency and quality of \method from random initialization at ImageNet scale remain untested.
Fourth, even at $\lambda=1$, representation-space properness concerns only projected distributions, so full image matching depends on the information retained by the encoder.
Fifth, the benchmark establishes sample quality at NFE${}=1$, but not compute-matched training efficiency against mature diffusion and autoregressive systems.
Finally, in strongly conditional settings where each condition is paired with only one real sample, the empirical conditional energy objective may encourage matching that single paired sample unless diversity is supplied through data augmentation, shared statistical structure across conditions, or additional regularization.

\section{Conclusion}
\label{sec:conclusion}

We presented \method as a direct particle-interaction paradigm for high-dimensional one-step generation without teacher predictions.
Starting from the energy distance, its core algorithm contrasts data--model attraction with model--model self-interaction in local three-body events, and our experiments take this rule from particle optimization to image generation.
Tracked scattering converts noisy instantaneous observations into an online learned field, while the design map clarifies connections to Drift-like dynamics, GAN-like critic fields, and representation-space objectives.

\bibliography{resources/reference}
\bibliographystyle{configuration/iclr2026_conference}

\appendix
\onecolumn
{
    \hypersetup{linkcolor=black}
    \parskip=0em
    \renewcommand{\contentsname}{Contents}
    \tableofcontents
    \addtocontents{toc}{\protect\setcounter{tocdepth}{3}}
}

\newpage

\begin{figure}[!t]
    \centering
    \begin{subfigure}[t]{0.485\textwidth}
        \centering
        \includegraphics[width=\linewidth]{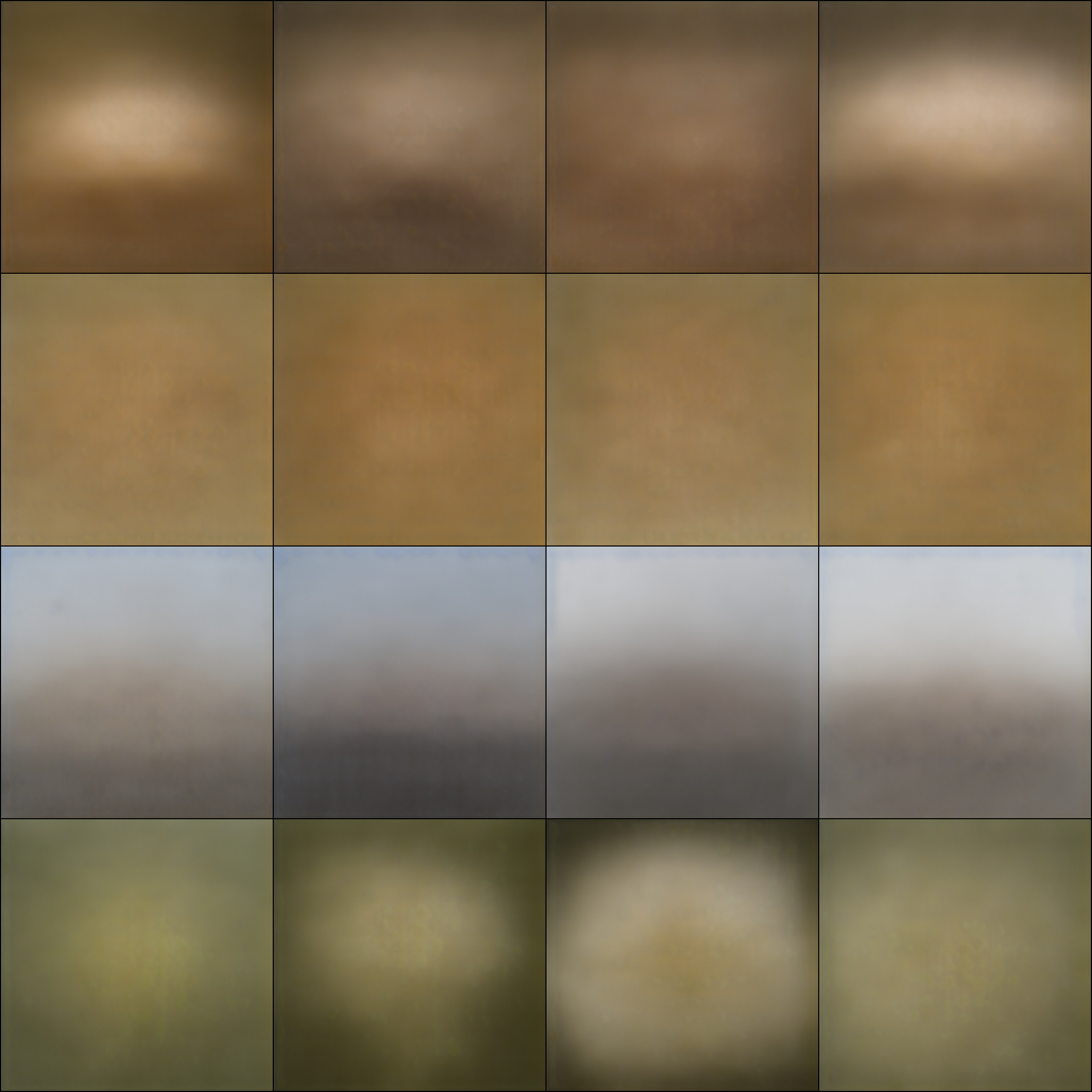}
        \caption{FID${}=398.20$\quad FDr$^6{}=291.50$\quad IS${}=1.17$.}
    \end{subfigure}
    \hfill
    \begin{subfigure}[t]{0.485\textwidth}
        \centering
        \includegraphics[width=\linewidth]{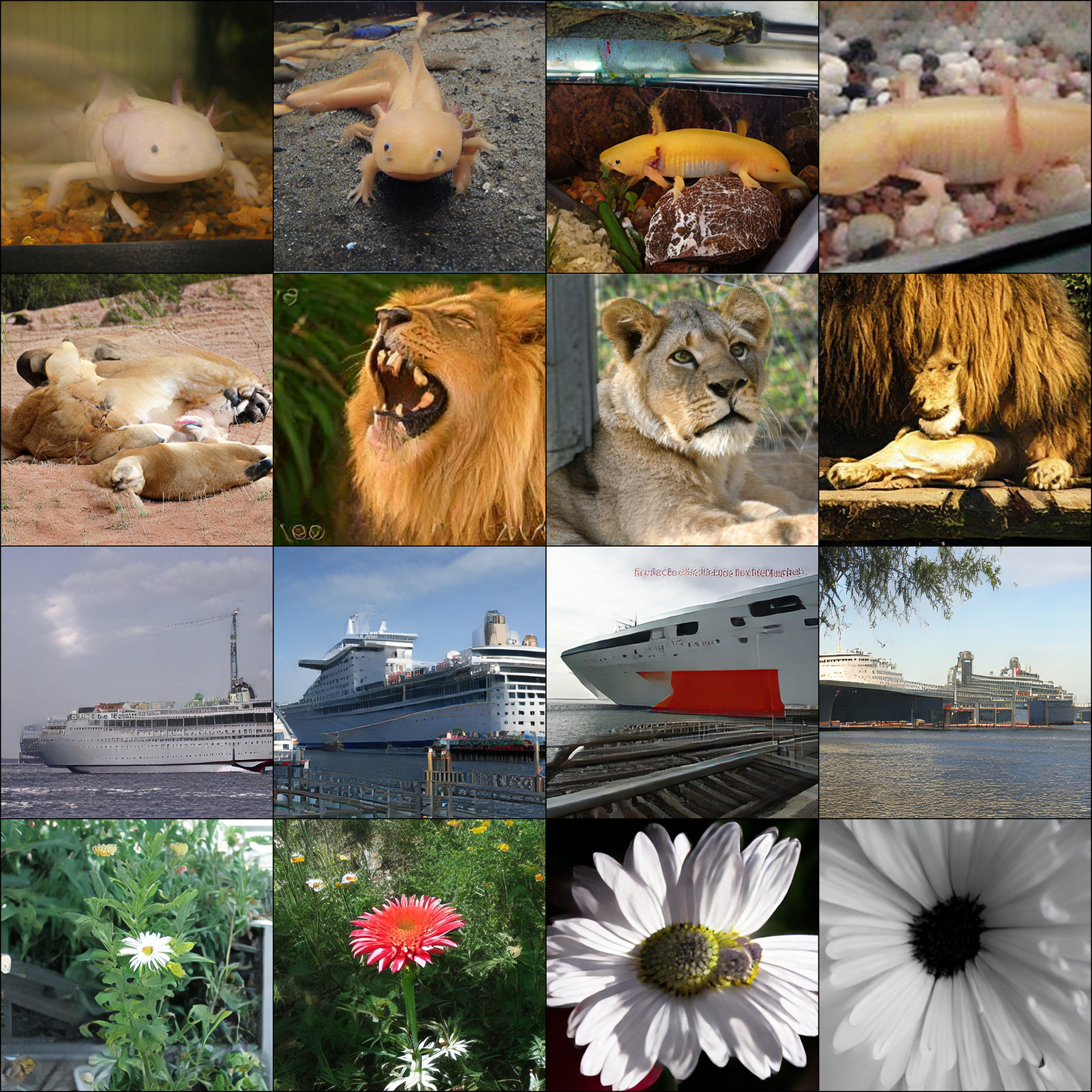}
        \caption{FID${}=1.92$\quad FDr$^6{}=5.37$\quad IS${}=252.1$.}
    \end{subfigure}
    \caption{
        \textbf{Converting a multi-step ImageNet-512 model into a one-step generator.}
        Single-pass samples from a pretrained multi-step DiT-XL/4 before (left) and after \method{} training with $\rho=\lambda=0.9$ (right).
        Both grids are random and uncurated.
    }
    \label{fig:imagenet512-conversion}

    \vspace{0.25em}

    \begin{subfigure}[t]{0.485\textwidth}
        \centering
        \includegraphics[width=\linewidth]{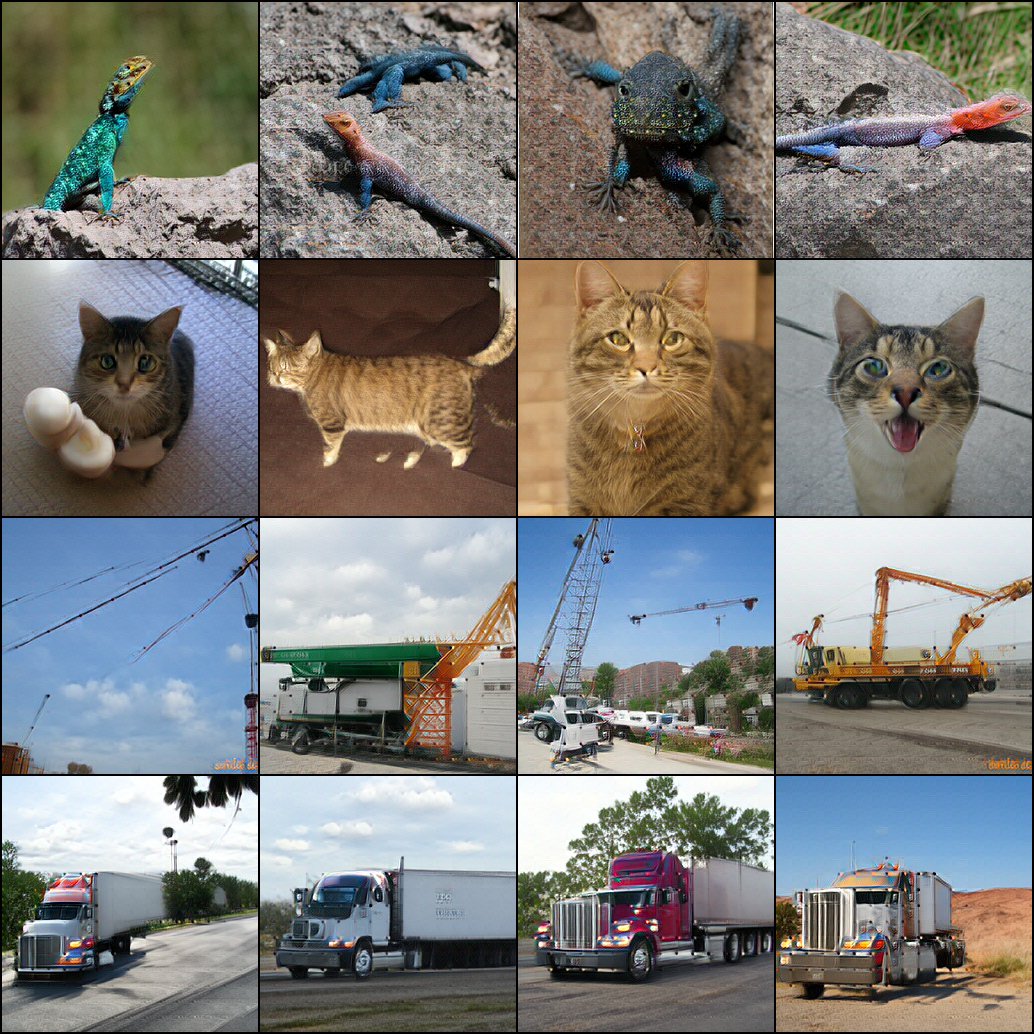}
        \caption{FID${}=0.99$\quad FDr$^6{}=5.64$\quad IS${}=328.0$.}
    \end{subfigure}
    \hfill
    \begin{subfigure}[t]{0.485\textwidth}
        \centering
        \includegraphics[width=\linewidth]{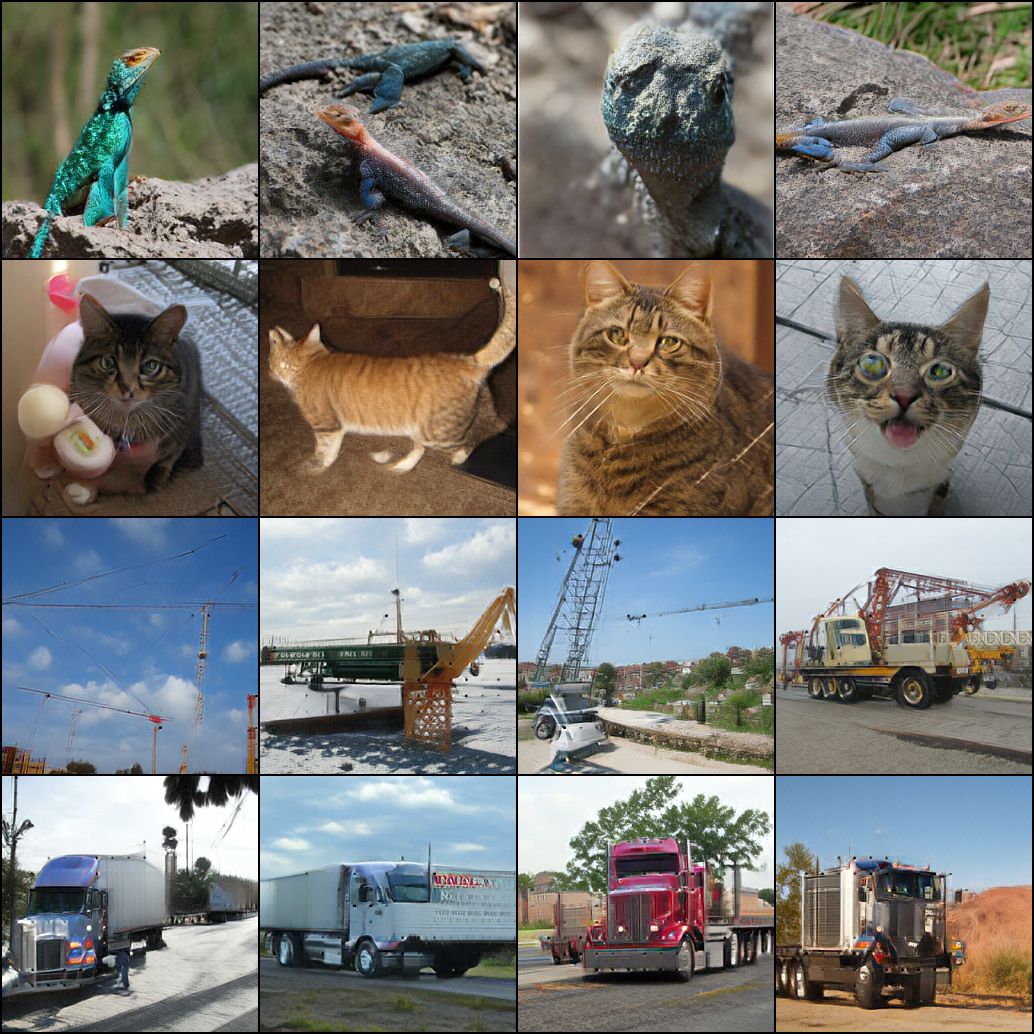}
        \caption{FID${}=1.84$\quad FDr$^6{}=8.34$\quad IS${}=247.3$.}
    \end{subfigure}
    \caption{
        \textbf{FD-loss-trained JiT-B checkpoint refinement on ImageNet-256.}
        Random, uncurated single-pass samples before (left) and after continued \method{} training with $\rho=\lambda=0.9$ (right).
    }
    \label{fig:fdloss-refinement}
\end{figure}

\section{Additional ImageNet Results}
\label{app:additional-image-results}

\tabref{tab:imagenet-training-config} gives optimizer, batch-size, and training-step settings for both studies.

\paragraph{Multi-step-to-one-step conversion at $512\times512$.}
The left panel of \figref{fig:imagenet512-conversion} evaluates a pretrained multi-step DiT-XL/4 at NFE${}=1$, outside its intended sampler; after \method{} training, the same architecture operates as an NFE${}=1$ ImageNet-512 generator.
For the first half of this high-resolution experiment, we use an empirical paired noisy-input curriculum with \algref{alg:tracked-three-body-training}: we sample $t\in\{0.5,1.0\}$ with equal probability, form $\xx_t=t\zz+(1-t)\xx_{\mathrm r}$, set $\xx_{\mathrm p}=\mg_{\mtheta}(\xx_t,t,\cc)$, and additionally condition the tracker on $t$.
Throughout, the generated source $\xx_{\mathrm s}=\sg[\mg_{\mtheta}(\zz_{\mathrm s},1,\cc)]$ uses an independent pure-noise draw $\zz_{\mathrm s}$ and does not depend on the sampled $t$.
When $t=0.5$, however, $\xx_{\mathrm p}$ is constructed from $\xx_{\mathrm r}$ and therefore does not satisfy real-source--projectile independence; moreover, $\xx_{\mathrm p}$ and $\xx_{\mathrm s}$ are not sampled iid from the same current generated law, as assumed by the exact field and convergence analysis.
We use this curriculum solely as an empirical stabilization heuristic; none of our theoretical claims relies on it.
For the second half, we fix $t=1.0$, restoring the pure-noise generator input used at one-step deployment.
Using the same optimization settings, PixelDiT-XL reaches FID${}=3.84$, IS${}=236.44$, and FDr$^6{}=8.17$ on ImageNet-512.

\paragraph{FD-loss checkpoint case study.}
In this separate JiT-B case, the post-training grid displays fewer block-like and grid-like artifacts although all three aggregate metrics worsen (\figref{fig:fdloss-refinement}; \citealp{yang2026representation}).
This illustrates a possible metric--artifact mismatch in the displayed samples, not a general limitation of FD-loss.

\FloatBarrier

\section{Implementation Details}
\label{app:exp}

\begin{table}[!t]
    \centering
    \caption{
        \textbf{ImageNet training configurations: shared optimizer settings and per-experiment batch sizes and steps.}
    }
    \scriptsize
    \setlength{\tabcolsep}{3pt}
    \renewcommand{\arraystretch}{1.08}
    \begin{tabular*}{\textwidth}{@{\extracolsep{\fill}}lccccc@{}}
        \toprule
        \textbf{Configuration} & \multicolumn{5}{c}{\textbf{Shared setting}} \\
        \midrule
        Generator optimizer & \multicolumn{5}{c}{AdamW} \\
        Generator learning rate & \multicolumn{5}{c}{$1.0\times10^{-5}$} \\
        Generator betas & \multicolumn{5}{c}{$(0.9, 0.999)$} \\
        Tracker optimizer & \multicolumn{5}{c}{AdamW} \\
        Tracker learning rate & \multicolumn{5}{c}{$1.0\times10^{-3}$} \\
        Tracker betas & \multicolumn{5}{c}{$(0.9, 0.999)$} \\
        Generator EMA decay & \multicolumn{5}{c}{$0.999$} \\
        Training precision & \multicolumn{5}{c}{BF16} \\
        \midrule
        \textbf{Configuration}
        & \shortstack{\textbf{JiT-B}\\\textbf{(\tabref{tab:imagenet-benchmark})}}
        & \shortstack{\textbf{PixelDiT-XL}\\\textbf{(\tabref{tab:imagenet-benchmark})}}
        & \shortstack{\textbf{DiT-XL/2}\\\textbf{(\tabref{tab:imagenet-benchmark})}}
        & \shortstack{\textbf{DiT-XL/4}\\\textbf{(\figref{fig:imagenet512-conversion})}}
        & \shortstack{\textbf{JiT-B}\\\textbf{(\tabref{tab:representation-ablation}; \figref{fig:lambda-rho-plane}, \figref{fig:fdloss-refinement})}}
        \\
        \midrule
        Global batch size & $256$ & $128$ & $256$ & $256$ & $256$ \\
        Training steps & $100{,}000$ ($200{,}000$ for $\dagger$) & $100{,}000$ & $50{,}000$ & $25{,}000$ & $5{,}000$ \\
        \bottomrule
    \end{tabular*}
    \label{tab:imagenet-training-config}
\end{table}

Each training step performs one generator update and, when $\rho>0$, one tracker update on the same triplet minibatch.

\paragraph{Numerical stabilization.}
All reported experiments compute each bearing as $(\yy-\xx)/(\norm{\yy-\xx}_2+\varepsilon)$ with fixed, untuned $\varepsilon=10^{-6}$ to prevent division by zero.
\secref{app:denominator-smoothing} characterizes the induced smoothed objective and its relation to the exact bearing field.

\paragraph{Tracker architecture.}
The tracker is a lightweight conditional Transformer with approximately $30$M parameters, comprising four layers with hidden dimension $736$, eight attention heads, and an MLP expansion ratio of $4$.
Input features are linearly embedded and augmented with two-dimensional sinusoidal positional encodings and a joint class--timestep conditioning embedding, then processed by four self-attention--MLP blocks with RMSNorm, residual connections, and spectral normalization.
Reported image experiments use a linear vector-valued head that projects the resulting tokens back to the input feature dimension.
The conservative-field alternative instead predicts a scalar potential per token; these potentials are summed, and differentiating the resulting scalar with respect to the input yields the vector field.

\paragraph{Representation interfaces.}
For latent-output models, each generated latent is decoded by the corresponding frozen decoder before the resulting image is passed through the frozen representation encoders used to compute the scattering loss; ImageNet DiT uses the SD-VAE.
On the projectile branch, gradients propagate through both decoder and encoders into the generator, while source features remain detached as in \algref{alg:tracked-three-body-training}.
Thus, ``latent'' and ``pixel'' in \tabref{tab:imagenet-benchmark} denote the generator output domain; scattering is computed in frozen image-representation spaces in both cases.
For every representation model, we use only its final-layer feature map; unlike Drifting Models~\citep{deng2026drifting}, which combines features from multiple stages, scales, and locations, we do not aggregate intermediate layers.
For a feature tensor $\mmF\in\R^{B\times C\times H\times W}$, each minibatch supplies channel-wise means and standard deviations over $(B,H,W)$ in $\R^{1\times C\times1\times1}$; we maintain detached exponential moving averages with decay $0.999$.
Each sample is whitened with the pre-update running statistics and then $\ell_2$-normalized along $C$ at each spatial location.
For scattering, we then identify each normalized $C\times H\times W$ sample with its vectorization in $\R^{CHW}$, so every bearing denominator uses the global Euclidean norm (equivalently, the tensor Frobenius norm) across all channels and spatial locations, not a positionwise norm.
Only after the losses are formed are the running statistics updated with the current minibatch.
Thus the preprocessing map is fixed conditional on the history before each loss event, so the current-step estimator and tracker identities remain valid conditionally on that history.
The theory below treats the vectors entering the scattering objective as the primitive sample space and does not cover the slow time variation induced by these EMA updates.
Each representation model has its own tracker; all representation branches use the same feature processing and tracker architecture, and their losses are summed only after being formed in their respective feature spaces.

\paragraph{Generator initialization.}
Except for the FD-loss checkpoint case in \figref{fig:fdloss-refinement}, all ImageNet-1K \method runs initialize the generator from an architecture- and resolution-matched multi-step checkpoint trained with a diffusion or flow-matching objective.
For JiT and PixelDiT, we use the authors' public checkpoints for their best-reported multi-step configurations~\citep{li2025back,yu2025pixeldit}, available from the official project repositories.
For DiT, we instead initialize DiT-XL/2 and DiT-XL/4 from the multi-step checkpoints released by UCGM~\citep{sun2025unified}.
The initialized network is then optimized with the relevant setting of \algref{alg:tracked-three-body-training}.
Unlike consistency distillation~\citep{song2023consistency}, which couples predictions across noise levels, and DMD~\citep{yin2024one}, which queries diffusion and generated-distribution scores, this warm start supplies only initial parameters: no consistency target, teacher prediction, score, or noise-time target enters the \method update.
Thus, the reported one-step models use \method supervision and a single-pass sampler; ``one-step'' does not imply random initialization.
For the primary runs, this practical warm start improves training efficiency while leaving the \method objective independent of the initialization source.
\tabref{tab:imagenet-training-config} reports the shared optimizer settings and per-experiment batch sizes and step counts for this \method training.

\FloatBarrier

\begin{figure}[!t]
    \centering
    \includegraphics[width=0.995\textwidth]{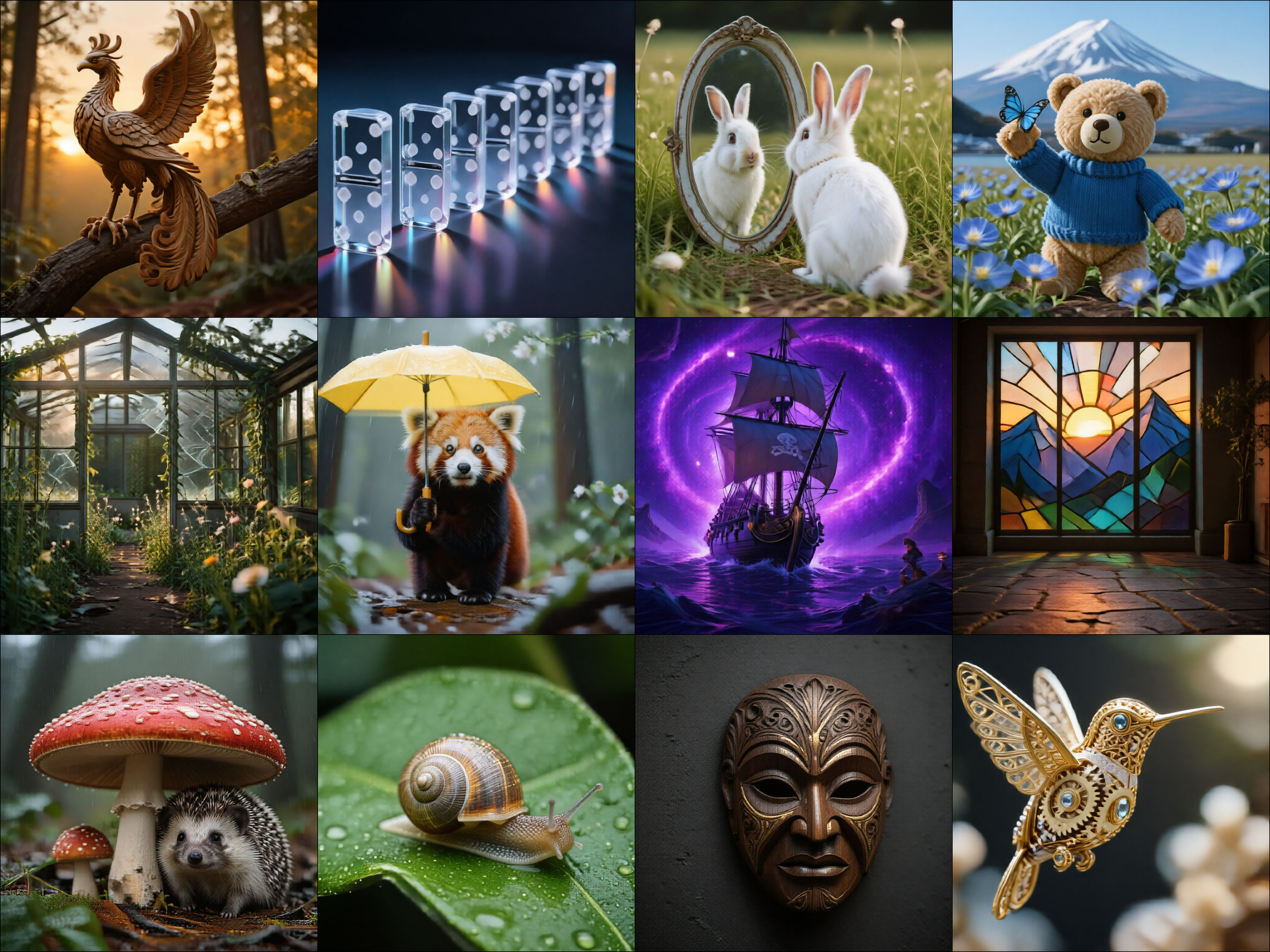}
    \caption{
        \textbf{One-step text-to-image generation.}
        NFE${}=1$ samples generated at $1024\times1024$ resolution by a \method{}-fine-tuned Qwen-Image-20B model.
    }
    \label{fig:qwen-image-grid}
\end{figure}

\section{Text-to-Image Generation}
\label{app:text-to-image}

\paragraph{Initialization and supervision.}
We initialize the generator directly from the original pretrained multi-step Qwen-Image-20B~\citep{wu2025qwenimagetechnicalreport} checkpoint, without an intermediate distillation stage or teacher-assisted conversion.
Thereafter, all generator and tracker updates use only \algref{alg:tracked-three-body-training}.
Relative to teacher-based distillation and adversarial post-training pipelines~\citep{yin2024one,lin2026continuousadversarial}, the supervision interface is simpler in this respect: it requires no path-model teacher query, score estimate, denoising-trajectory target, or adversarial critic.

\paragraph{Model and representation interfaces.}
The samples in \figref{fig:scale-overview} and \figref{fig:qwen-image-grid} are generated at NFE${}=1$ without classifier-free guidance.
Relative to the ImageNet implementation, each tracker introduces cross-attention to the text-conditioning sequence from the Qwen2.5-VL~\citep{bai2025qwen25vl} text encoder in Qwen-Image-20B.
We form separate scattering losses from the final-layer features of three frozen representation models: the same Qwen2.5-VL encoder, DINOv3~\citep{simeoni2025dinov3}, and SigLIP~\citep{zhai2023sigmoid}.
Each representation has its own tracker and follows the whitening, normalization, and loss-aggregation procedure in \secref{app:exp}.

\paragraph{Training configuration.}
We set $\rho=1$ and $\lambda=0$ and train for $1{,}000$ generator updates with batch size $16$.
The generator uses learning rate $5\times10^{-6}$ and AdamW betas $(0.9, 0.99)$; each tracker uses learning rate $1\times10^{-5}$ and betas $(0.9, 0.999)$.
Training likewise uses no classifier-free guidance.

\paragraph{Evaluation scope.}
The current draft reports this initial $1{,}000$-step study.
Subsequent updates will evaluate additional generator families, longer training, and broader benchmarks with direct baseline comparisons.

\FloatBarrier

\section{Related Work and Positioning}
\label{app:positioning}

We position \method from its exact energy-distance endpoint: an explicit distributional energy whose population field is estimated with constant-size stochastic interactions, without adversarial, path-indexed, autoregressive, or distillation-based supervision.

\subsection{Adversarial One-Step Generators}
\label{app:positioning-gan}

GANs~\citep{goodfellow2014generative,brock2018large,sauer2022stylegan,kang2023scaling} provide direct single-pass sampling but train the generator through a jointly learned discriminator or critic.
Adversarial Flow Models use adversarial training for native one- or multi-step flow-style generation, while Continuous Adversarial Flow Models adversarially post-train continuous flow models~\citep{lin2025adversarialflow,lin2026continuousadversarial}.
\method instead uses non-adversarial frozen-target regression; at its exact endpoint, the sampled scattering vector estimates the $2$-Wasserstein negative-gradient field of the energy-distance functional.
A GAN-like objective arises only for the conservative displacement analogue in \secref{sec:path-conditioned-fallback}, where \eqref{eq:gan-like-critic-objective} gives a quadratically path-regularized potential objective.

\subsection{Diffusion, Score, and Flow-Based Models}
\label{app:positioning-diffusion-flow}

Standard diffusion, score-based, and flow-matching models~\citep{ho2020denoising,song2020score,song2020denoising,lipman2022flow,ma2024sit} use path-indexed supervision: a network predicts a denoising target, score, or velocity along a noise-to-data trajectory and is typically evaluated repeatedly during sampling.
The exact energy-distance formulation requires no noise schedule or transport path; it updates the generator distribution directly through real-generated-generated particle interactions.

\subsection{Autoregressive Visual Generation}
\label{app:positioning-ar}

Autoregressive visual generators~\citep{oord2016pixel,yu2023language,li2024autoregressive,tian2024visual,sun2024autoregressive} formulate image generation as ordered conditional prediction over discrete or continuous token, patch, or scale representations.
They exploit sequence-modeling infrastructure, while sampling follows a chosen order or scale schedule; parallel and masked variants can reduce the number of decoding rounds.
\method is non-autoregressive: it emits a full continuous sample in one generator evaluation and trains it through the population scattering vector field.

\subsection{One-Step and Few-Step Diffusion Acceleration}
\label{app:positioning-acceleration}

The sampling cost of diffusion and flow models has motivated a large body of one-step and few-step acceleration methods, including progressive distillation~\citep{salimans2022progressive}, consistency models~\citep{song2023consistency,lu2024simplifying,song2023improved}, shortcut models~\citep{frans2024one}, mean flows~\citep{geng2025mean}, and distribution matching distillation~\citep{yin2024one,yin2024improved}.
They obtain one- or few-step sampling through method-specific combinations of path- or noise-time supervision, consistency constraints, teachers, or auxiliary distribution estimators.
At its energy-distance endpoint, \method instead trains the one-step generator with a population distribution-matching objective that does not require teacher outputs, trajectory targets, or score estimates.
Our primary ImageNet experiments use a multi-step checkpoint only for parameter initialization, as detailed in \secref{app:exp}; this practical warm start is distinct from making a path model part of the \method supervision.

\subsection{Generative Wasserstein Flows and Particle Dynamics}
\label{app:positioning-wgf}

The closest line directly optimizes distributional discrepancies or their induced dynamics.
Energy distance is equivalent to MMD with a distance-induced kernel~\citep{sejdinovic2013equivalence}; both yield proper population objectives under their standard assumptions~\citep{szekely2013energy,gretton2012kernel}.
Generative Moment Matching Networks directly backpropagate an MMD objective~\citep{li2015generative}, whereas MMD-GAN and Cram\'er GAN learn the comparison kernel or feature map adversarially~\citep{li2017mmdgan,bellemare2017cramer}.
Coulomb GAN uses a learned potential field with data attraction and generated-particle repulsion~\citep{unterthiner2018coulomb}.
Under the normalization below, the $\lambda=1$ population field in \eqref{eq:population-field} is exactly the velocity field of the $2$-Wasserstein gradient flow of distance-kernel $\mathrm{MMD}^2$; alternative constant normalizations only rescale time.
MMD Wasserstein flows have been developed for interacting-particle transport, one-dimensional negative-distance dynamics, sliced Riesz kernels, posterior sampling, and deep generative modeling~\citep{arbel2019maximum,duong2024distance,hertrich2024generative,hagemann2023posterior,galashov2025deep}.
Smoothed distance kernels have also been constructed to restore differentiability and flow well-posedness near particle collisions~\citep{rux2026smoothed}.
Recent work strengthens this theory in complementary directions: Sobolev-regularized MMD flow penalizes witness gradients and proves global MMD convergence under embedding-regularity assumptions, while kernel-gradient drifting expresses general-kernel drifts as score differences between kernel-smoothed laws, yielding a smoothed-KL descent interpretation~\citep{tian2026sobolev,esteban2026kernel}.
The population energy flow is therefore not claimed as new; our theory concerns its constant-size source-detached estimator, online conditional-expectation tracking, and finite-error transfer to generator updates.
More broadly, Wasserstein gradient-flow theory and Stein variational methods provide related distributional and interacting-particle views~\citep{jordan1998variational,ambrosio2008gradient,liu2016stein}.
MMD admits both full pairwise empirical estimates and linear-time random-pair approximations~\citep{gretton2012kernel}.

\paragraph{Relation to linear-time MMD estimation.}
\label{app:linear-mmd-relation}
For distributions with finite first moments and any anchor $\aa$, the distance-induced kernel
$k_{\aa}(\xx,\yy)=\frac12(\norm{\xx-\aa}+\norm{\yy-\aa}-\norm{\xx-\yy})$
makes $\frac{1}{2}D_E^2(P_{\mtheta},Q)$ equal to the squared MMD.
For independent $\xx_1,\xx_2\sim P_{\mtheta}$ and $\xx_{\mathrm r,1},\xx_{\mathrm r,2}\sim Q$, one block of its classical linear-time estimator and its generator gradient can be written as
\begin{equation}
    \begin{aligned}
        \widehat{\cF}_{\mathrm{lin}}
         & =
        \frac{1}{2}
        \left(
        \norm{\xx_1-\xx_{\mathrm r,2}}
        +\norm{\xx_2-\xx_{\mathrm r,1}}
        -\norm{\xx_1-\xx_2}
        -\norm{\xx_{\mathrm r,1}-\xx_{\mathrm r,2}}
        \right),
        \\
        \nabla_{\mtheta}\widehat{\cF}_{\mathrm{lin}}
         & =
        -\frac{1}{2}
        \left(
        J_1^{\top}\widehat{\vv}_{\mathrm{scat}}(\xx_1;\xx_{\mathrm r,2},\xx_2)
        +
        J_2^{\top}\widehat{\vv}_{\mathrm{scat}}(\xx_2;\xx_{\mathrm r,1},\xx_1)
        \right),
    \end{aligned}
    \label{eq:linear-mmd-three-body-gradient}
\end{equation}
where $J_i=\nabla_{\mtheta}\xx_i$ and the identity holds away from coincident samples.
Thus a symmetrized linear-time distance-kernel MMD gradient decomposes into two three-body projectile updates.
By exchangeability, sampling either projectile term uniformly and omitting its factor $1/2$ has the same expectation as the displayed average.
\algref{alg:three-body-training} realizes this update through regression with the generated source detached; up to conventional loss scaling, \thmref{thm:three-body-equivalence} proves equality with the energy-distance gradient at the current parameters.
Thus \algref{alg:three-body-training} is a source-detached projectile-level realization of a classical stochastic gradient, not a new linear-time MMD estimator.

Recent direct distribution-training methods differ mainly in how they estimate distribution-level supervision.
Drifting Models construct minibatch kernel attraction--repulsion fields, with subsequent work analyzing when such fields are conservative~\citep{deng2026drifting,franz2026drifting}; FD-loss instead optimizes representation-space feature means and covariances using population statistics and batch-level gradients~\citep{yang2026representation}.
Representation Distribution Matching uses a frozen Nystr\"om data reference but retains exact within-batch generated-sample repulsion, making the fresh generated batch an operative optimization variable~\citep{feng2026representation}.
Variational generative Wasserstein flows use JKO or dual formulations~\citep{caucheteux2026unifying}, whereas Sinkhorn-Drifting and W-Flow construct minibatch cross- and self-transport plans~\citep{he2026sinkhorndrifting,han2026wflow}.
W-Flow is especially close at the training interface: it regresses a one-step generator toward detached particle updates obtained from minibatch Sinkhorn plans.
At its exact endpoint, \method estimates each projectile's energy-distance field from one real and one independently generated source.
A batch of $B$ projectiles therefore uses $O(B)$ sampled source interactions and explicit frozen targets, without solving a minibatch transport problem; tracked scattering learns their conditional expectation online.
Operationally, this sample-level interface requires one observed real source per condition-specific loss rather than a same-condition real minibatch, fitting text--image data with one image per caption.
It does not remove the statistical limitation of observing one sample from each conditional law, discussed in \secref{sec:limitations}.

\paragraph{Guidance-free training interface.}
Classifier-free guidance is not intrinsic to the neighboring families, but leading implementations may build it into training: Drifting samples a guidance scale and unconditional real negatives, although its best reported FID uses the scale it identifies as ``no CFG''; W-Flow adopts velocity guidance as its default; and text-to-image DMD distills teachers at fixed guidance scales~\citep{deng2026drifting,han2026wflow,yin2024one,ho2022classifier}.
The reported \method pipeline uses neither unconditional reference particles nor a guided teacher or sampler and introduces no separate guidance hyperparameter beyond its intrinsic $(\rho,\lambda)$ controls.
This is a claim about guidance-related pipeline simplicity, not that these methods cannot operate without CFG.

\section{Derivation and Theoretical Foundations}
\label{app:theory}

Unless stated otherwise, distribution-level identities are for one fixed condition.
They extend by averaging over $\cc$ when the required expectations and interchanges are valid; condition-wise rates require corresponding integrable or uniform control.
Parameter-space results apply to the aggregate conditional objective only when all their stated assumptions hold for it.
Throughout this section, $\xx\in\R^d$ denotes the vector supplied to the scattering objective.
We assume finite first moments and set every bearing to zero at exact coincidence.
Classical spatial and parameter gradients additionally require zero collision probability and valid differentiation under the expectation; with collision mass, zero is only a generalized-gradient selection.
Population $2$-Wasserstein flow statements further assume $P_t\in\mathcal P_2(\R^d)$ solves the continuity equation regularly enough for the $2$-Wasserstein chain rule.
At $\lambda=1$, the code's denominator smoothing retains a proper objective; the frozen-target, tracker-error, and conditional convergence arguments apply to its corresponding smoothed field as detailed in \secref{app:denominator-smoothing}.

\subsection{Exact Scattering Field and Frozen-Target Equivalence}
\label{app:three-body-derivation}

For one fixed condition, or in the unconditional case, write $P=P_{\mtheta}$ and retain $Q$ for the corresponding data law; the energy from the main text is
\begin{equation}
    \cF(P)
    =\Eb{\norm{\xx_{\mathrm p}-\xx_{\mathrm r}}}
    -\frac12\Eb{\norm{\xx_{\mathrm p}-\xx_{\mathrm s}}}
    -\frac12\Eb{\norm{\xx_{\mathrm r}-\xx_{\mathrm r}'}} ,
    \label{eq:app-energy-functional}
\end{equation}
where $\xx_{\mathrm p},\xx_{\mathrm s}\sim P$ and $\xx_{\mathrm r},\xx_{\mathrm r}'\sim Q$ are independent within and across laws.
The real--real term is constant in $P$, while symmetry of the model--model term removes its factor $1/2$ under variation.
Thus, up to an additive constant,
\begin{align}
    \frac{\delta \cF}{\delta P}(\xx)
     & =\Eb{\norm{\xx-\xx_{\mathrm r}}}-\Eb{\norm{\xx-\xx_{\mathrm s}}},
    \label{eq:app-first-variation}
    \\
    \nabla_{\xx}
    \frac{\delta \cF}{\delta P}(\xx)
     & =\Eb{\frac{\xx-\xx_{\mathrm r}}{\norm{\xx-\xx_{\mathrm r}}}}
    -\Eb{\frac{\xx-\xx_{\mathrm s}}{\norm{\xx-\xx_{\mathrm s}}}},
    \label{eq:app-spatial-gradient}
    \\
    \vv(\xx)
     & =-\nabla_{\xx}\frac{\delta \cF}{\delta P}(\xx)
    =\Eb{\frac{\xx_{\mathrm r}-\xx}{\norm{\xx_{\mathrm r}-\xx}}}
    -\Eb{\frac{\xx_{\mathrm s}-\xx}{\norm{\xx_{\mathrm s}-\xx}}}.
    \label{eq:app-population-velocity}
\end{align}
Evaluating the last line at a projectile and sampling one source from each law gives \eqref{eq:scattering-direction}.

\paragraph{Frozen-target pullback.}
\label{app:gradient-surrogate}
For directly optimized particles, descending
$\frac12\norm{\xx_{\mathrm p}-\sg[\xx_{\mathrm p}+\widehat{\vv}_{\mathrm{scat}}]}_2^2$
takes the Euler step $\xx_{\mathrm p}\leftarrow\xx_{\mathrm p}+\eta\widehat{\vv}_{\mathrm{scat}}$.
For a generator, detach fixes this sampled field and the chain rule pulls it back through the projectile Jacobian.

\begin{proof}[Proof of Theorem~\ref{thm:three-body-equivalence}]
    Independence of both sources from the projectile branch gives
    $\Eb{\widehat{\vv}\mid\xx_{\mathrm p},\cc}=\vv_{\bar{\mtheta}}(\xx_{\mathrm p},\cc)$ and, for the pullback, the same identity conditional on the projectile latent and its Jacobian; also $\norm{\widehat{\vv}}_2\leq2$.
    Under the regularity assumptions in the theorem, all target-side quantities are fixed at $\bar{\mtheta}$, so at the current parameters
    \begin{equation}
        \left.\nabla_{\mtheta}\cL_{\mathrm{gen},\cc}(\mtheta;\bar{\mtheta})
        \right|_{\mtheta=\bar{\mtheta}}
        =-\Eb{J_{\bar{\mtheta}}^{\top}\widehat{\vv}}
        =-\Eb{J_{\bar{\mtheta}}^{\top}\vv_{\bar{\mtheta}}}
        =\nabla_{\mtheta}\cF_{\cc}(\bar{\mtheta}).
        \label{eq:app-frozen-target-gradient-equals-energy-gradient}
    \end{equation}
    The middle equality uses conditional unbiasedness, and the last is \eqref{eq:generator-gradient-field}.
    Without detach, derivatives through the generated source and target construction would add terms and change this identity.
\end{proof}
Dropping the conventional factor $1/2$ doubles the regression gradient and is absorbed by the learning rate.

\subsection{Tracker Accuracy and Bias--Variance}
\label{app:tracker-conditional-expectation}

Let $\mathcal H$ denote the training history before a fresh triplet is drawn, and evaluate the current tracker before updating it on that triplet, so the current generator and tracker are fixed conditional on $\mathcal H$.
Write $\mathbf q=(\tilde{\xx},\cc)$, $Y=\widehat{\vv}_{\lambda}$, $m_{\lambda}(\mathbf q)=\Eb{Y\mid\mathbf q,\mathcal H}$, and $h_{\mphi}(\mathbf q)=\mv_{\mphi}(\tilde{\xx},\cc)$.
Assume $h_{\mphi}(\mathbf q)$ is square integrable and define $G_{\rho}=(1-\rho)Y+\rho h_{\mphi}(\mathbf q)$ for $\rho\in[0,1]$.
The population tracker risk is $R(h)=\Eb{\norm{h(\mathbf q)-Y}_2^2\mid\mathcal H}$.

\begin{theorem}[Tracked scattering as a conditional-expectation projection]
    \label{thm:finite-tracker-decomposition}
    Over all square-integrable functions measurable with respect to the query and history, the conditional expectation $m_{\lambda}$ minimizes $R$, and the excess risk of the learned tracker is $R(h_{\mphi})-R(m_{\lambda})=\Eb{\norm{h_{\mphi}(\mathbf q)-m_{\lambda}(\mathbf q)}_2^2\mid\mathcal H}$.
    Moreover, almost surely, the mixed target has the exact error decomposition
    \begin{equation}
        \Eb{
            \norm{
                G_{\rho}
                -
                m_{\lambda}(\mathbf q)
            }_2^2
            \mid\mathbf q,\mathcal H
        }
        =
        (1-\rho)^2
        \operatorname{tr}\operatorname{Var}
        (Y\mid\mathbf q,\mathcal H)
        +
        \rho^2
        \norm{h_{\mphi}(\mathbf q)-m_{\lambda}(\mathbf q)}_2^2.
        \label{eq:app-finite-tracker-decomposition}
    \end{equation}
    At $\lambda=1$, $\mathbf q=(\xx_{\mathrm p},\cc)$ and $m_1(\mathbf q)=\vv_{\mtheta}(\xx_{\mathrm p},\cc)$.
    In this case, the expected tracked velocity remains a strict energy-descent direction whenever $\rho\norm{h_{\mphi}-\vv_{\mtheta}}_{L^2(P_{\mtheta})}<\norm{\vv_{\mtheta}}_{L^2(P_{\mtheta})}$, condition-wise.
\end{theorem}

\begin{proof}
    The conditional expectation is the $L^2$ orthogonal projection of $Y$ onto the sigma-algebra generated by $\mathbf q$ together with $\mathcal H$.
    Hence, for every square-integrable $h(\mathbf q)$, the cross term between $Y-m_{\lambda}(\mathbf q)$ and $h(\mathbf q)-m_{\lambda}(\mathbf q)$ vanishes conditional on $\mathcal H$.
    Expanding the tracker risk proves the minimizer and excess-risk statements.

    For the mixed target,
    $G_{\rho}-m_{\lambda}(\mathbf q)=(1-\rho)(Y-m_{\lambda}(\mathbf q))+\rho(h_{\mphi}(\mathbf q)-m_{\lambda}(\mathbf q))$.
    The same conditional orthogonality removes the cross term, proving \eqref{eq:app-finite-tracker-decomposition}.

    At $\lambda=1$, the tracker query is the projectile, and \thmref{thm:three-body-equivalence} identifies $m_1$ with the exact population scattering vector field $\vv_{\mtheta}$.
    Suppress the condition and write $\mepsilon_{\mphi}=h_{\mphi}-\vv_{\mtheta}$.
    The conditional expectation of $G_{\rho}$ is then $\vv_{\mtheta}+\rho\mepsilon_{\mphi}$.
    Whenever the Wasserstein chain rule is valid, the corresponding population flow satisfies
    \begin{equation}
        \frac{\dm}{\dm t}\cF(P_t)
        =
        -\norm{\vv_t}_{L^2(P_t)}^2
        -\rho\langle\vv_t,\mepsilon_{\mphi,t}\rangle_{L^2(P_t)}
        \leq
        -\norm{\vv_t}_{L^2(P_t)}
        \left(
        \norm{\vv_t}_{L^2(P_t)}
        -\rho\norm{\mepsilon_{\mphi,t}}_{L^2(P_t)}
        \right),
    \end{equation}
    where the inequality is Cauchy--Schwarz.
\end{proof}

At $\lambda=1$, full oracle tracking replaces $Y$ by $m_1=\Eb{Y\mid\xx_{\mathrm p},\cc,\mathcal H}$, its Rao--Blackwellization with respect to $\sigma(\xx_{\mathrm p},\cc,\mathcal H)$; this preserves the energy field and removes source-sampling variance but not projectile-sampling variance.

For any $\lambda\in[0,1]$, define the projectile-conditioned source expectation
$\overline{\vv}_{\lambda}(\xx_{\mathrm p},\cc)=\Eb{Y\mid\xx_{\mathrm p},\cc,\mathcal H}$, which is the negative Wasserstein-gradient field of \eqref{eq:lambda-weighted-functional} before tracker conditioning.
Then
\begin{equation}
    \Eb{G_{\rho}\mid\xx_{\mathrm p},\cc,\mathcal H}
    =
    (1-\rho)\overline{\vv}_{\lambda}(\xx_{\mathrm p},\cc)
    +\rho\Eb{h_{\mphi}(\mathbf q)\mid\xx_{\mathrm p},\cc,\mathcal H}.
    \label{eq:app-projectile-conditioned-mixed-mean}
\end{equation}
For $\lambda<1$, the query- and projectile-generated sigma-algebras are generally not nested, so even for an oracle tracker $h_{\mphi}=m_{\lambda}$, the second conditional expectation in \eqref{eq:app-projectile-conditioned-mixed-mean} need not equal $\overline{\vv}_{\lambda}$.
At $\lambda=1$, $\mathbf q=(\xx_{\mathrm p},\cc)$, so the two expectations coincide for an oracle tracker.
Thus \eqref{eq:app-finite-tracker-decomposition} remains exact around $m_{\lambda}$ for every $\lambda$, but for $\lambda<1$ it does not imply descent of either the energy distance or $\cF_{\lambda}$ after generator pullback.

If mixing can depend on the query, \eqref{eq:app-finite-tracker-decomposition} is minimized pointwise by
$\rho^\star(\mathbf q)=s_{\lambda}(\mathbf q)/(s_{\lambda}(\mathbf q)+e_{\mphi}(\mathbf q))$ whenever the denominator is positive, where
$s_{\lambda}(\mathbf q)=\operatorname{tr}\operatorname{Var}(Y\mid\mathbf q,\mathcal H)$ and
$e_{\mphi}(\mathbf q)=\norm{h_{\mphi}(\mathbf q)-m_{\lambda}(\mathbf q)}_2^2$.
For the query-independent weight used in the algorithm, let $S_{\lambda}=\Eb{s_{\lambda}(\mathbf q)\mid\mathcal H}$ and $E_{\mphi}=\Eb{e_{\mphi}(\mathbf q)\mid\mathcal H}$.
When $S_{\lambda}+E_{\mphi}>0$, integrating \eqref{eq:app-finite-tracker-decomposition} and minimizing its quadratic in $\rho$ gives $\rho^\star_{\mathrm{const}}=S_{\lambda}/(S_{\lambda}+E_{\mphi})$ and minimum error $S_{\lambda}E_{\mphi}/(S_{\lambda}+E_{\mphi})$.
Thus full tracking ($\rho=1$) improves on the instant target, relative to $m_{\lambda}$, precisely when its excess regression risk $E_{\mphi}$ is below the residual target variance $S_{\lambda}$.
At $\lambda=1$, this residual is exactly the source-sampling noise around the energy-distance field.
At this endpoint, \eqref{eq:app-finite-tracker-decomposition} specializes to \corref{cor:finite-tracker-error}; in the oracle case, $E_{\mphi}=0$ and $\rho^\star_{\mathrm{const}}=1$.

\paragraph{A two-draw tracker diagnostic.}
Fix $\lambda=1$ and condition on a history $\mathcal H$.
For each projectile query $\mathbf q=(\xx_{\mathrm p},\cc)$, hold the generator, tracker, query, and $h_{\mphi}(\mathbf q)$ fixed while independently resampling two complete source pairs, producing conditionally iid targets $Y_1,Y_2$.
Then
\begin{align}
    S_1
     & =\frac12\Eb{\norm{Y_1-Y_2}_2^2\mid\mathcal H},
     &
    E_{\mphi}
     & =\Eb{\left\langle h_{\mphi}(\mathbf q)-Y_1,
        h_{\mphi}(\mathbf q)-Y_2\right\rangle\mid\mathcal H}.
    \label{eq:app-two-draw-diagnostic}
\end{align}
Indeed, conditional on $(\mathbf q,\mathcal H)$, writing $Y_i=m_1(\mathbf q)+\xi_i$ with independent centered residuals gives the two identities after expansion and averaging over $\mathbf q$.
Consequently, when $S_1+E_{\mphi}>0$, the field-MSE-optimal constant mixture is
$\rho^\star_{\mathrm{const}}=S_1/(S_1+E_{\mphi})$, and full tracking has lower field MSE than the instant target exactly when $E_{\mphi}<S_1$.
Neither identity requires the unknown field $m_1$, so when the required conditional redraws are available, paired Monte Carlo averages provide unbiased snapshot estimates of $S_1$ and $E_{\mphi}$.
The resulting ratio is only a plug-in diagnostic (and may require clipping at finite sample size), not an unbiased estimate of $\rho^\star_{\mathrm{const}}$ or a certificate of generator convergence or sample quality.

At $\lambda=1$, the same orthogonal decomposition holds after the field is pulled back through the generator.
Let $J=\nabla_{\mtheta}\mg_{\mtheta}(\zz,\cc)$ and $\vv(\mathbf q)=m_1(\mathbf q)$.
Assume $J^{\top}(Y-\vv(\mathbf q))$ and $J^{\top}(h_{\mphi}(\mathbf q)-\vv(\mathbf q))$ are square integrable, and define, conditional on $\mathcal H$,
$S_J=\Eb{\norm{J^{\top}(Y-\vv(\mathbf q))}_2^2\mid\mathcal H}$ and
$E_{J,\mphi}=\Eb{\norm{J^{\top}(h_{\mphi}(\mathbf q)-\vv(\mathbf q))}_2^2\mid\mathcal H}$.

\Needspace{0.10\textheight}
\begin{corollary}[Generator-update tracker decomposition]
    \label{cor:parameter-tracker-decomposition}
    If $\Eb{Y-\vv(\mathbf q)\mid\mathbf q,J,\mathcal H}=0$, as ensured by the independent projectile and source draws in \algref{alg:tracked-three-body-training}, then
    \begin{equation}
        \Eb{
            \norm{J^{\top}(G_{\rho}-\vv(\mathbf q))}_2^2
            \mid\mathcal H
        }
        =
        (1-\rho)^2S_J+\rho^2E_{J,\mphi}.
        \label{eq:app-parameter-tracker-decomposition}
    \end{equation}
\end{corollary}

\begin{proof}
    Conditional on $(\mathbf q,J,\mathcal H)$, both $J$ and $h_{\mphi}(\mathbf q)-\vv(\mathbf q)$ are fixed.
    The assumed conditional-expectation identity makes the cross term vanish after multiplication by $J^{\top}$, proving \eqref{eq:app-parameter-tracker-decomposition}.
\end{proof}

Thus, when $S_J+E_{J,\mphi}>0$, the query-independent mixture minimizing the per-sample Jacobian-pullback error is
$\rho_J^\star=S_J/(S_J+E_{J,\mphi})$; full tracking improves on the instant update exactly when $E_{J,\mphi}<S_J$.
This parameter-space criterion can differ from the output-space criterion because the generator Jacobian suppresses or amplifies different field directions.

\subsection{Population Dissipation and Idealized Optimization Guarantees}

\subsubsection{Exact-flow dissipation}
\label{app:population-dissipation}
Let $P_t$ solve $\partial_t P_t+\nabla\cdot(P_t\vv_t)=0$ with
$\vv_t=-\nabla(\delta\cF/\delta P_t)$.

Whenever the $2$-Wasserstein chain rule is valid,
\begin{equation}
    \frac{\dm}{\dm t}\cF(P_t)
    =-\int\norm{\vv_t(\xx)}_2^2P_t(\dm\xx)\leq0.
    \label{eq:app-energy-dissipation}
\end{equation}
Indeed, writing $\phi_t=\delta\cF/\delta P_t$, the chain rule, continuity equation, and integration by parts give
$\frac{\dm}{\dm t}\cF(P_t)=\int\nabla\phi_t^{\top}\vv_t\,\dm P_t=-\int\norm{\vv_t}_2^2\,\dm P_t$.
Consequently, for every $T>0$,
\begin{equation}
    \inf_{0\leq t\leq T}\int\norm{\vv_t}_2^2\,\dm P_t
    \leq\frac1T\int_0^T\!\int\norm{\vv_t}_2^2\,\dm P_t\,\dm t
    \leq\frac{\cF(P_0)}{T}.
    \label{eq:app-population-stationarity-rate}
\end{equation}
This is field stationarity, not distributional or finite-generator convergence.
Properness by itself is insufficient to rule out nonoptimal stationary measures, so the slope condition used below is substantive.
A global slope inequality is not automatic because kernel discrepancies are generally not geodesically convex in Wasserstein geometry.
Related Riesz-flow theory proves global exponential convergence in the Coulomb regime under regularity assumptions~\citep{boufadene2025global} and local polynomial guarantees in higher-order Sobolev regimes under stronger assumptions~\citep{chizat2026quantitative}; neither directly supplies the energy-distance slope inequality used next.

\subsubsection{Tracked population flow}
\label{app:tracked-flow-convergence}

We now close the population-level chain for the exact energy-distance endpoint $\lambda=1$.
Let $Q$ be fixed, write $\cF(P)=\frac12D_E^2(P,Q)$, and let
$\vv_t=-\nabla(\delta\cF/\delta P_t)$ be its exact scattering field.
For a fixed $\rho\in[0,1]$, let $h_t$ be the current population tracker, set $e_t=h_t-\vv_t$, and suppose $P_t$ solves the continuity equation with the mean mixed velocity
$u_t=(1-\rho)\vv_t+\rho h_t=\vv_t+\rho e_t$.
Let $R_t$ denote the population tracker risk from \secref{app:tracker-conditional-expectation} at $P_t$.
The conditional-expectation projection gives
$R_t(h_t)-R_t(\vv_t)=\norm{e_t}_{L^2(P_t)}^2$.
In addition to the chain-rule regularity assumed above, suppose that along this trajectory there are constants $\gamma\in[0,1)$, $c>0$, and $q\geq1$ such that
\begin{equation}
    \rho\sqrt{R_t(h_t)-R_t(\vv_t)}
    \leq
    \gamma\norm{\vv_t}_{L^2(P_t)},
    \qquad
    \norm{\vv_t}_{L^2(P_t)}^2
    \geq
    c\,\cF(P_t)^q
    \quad\text{for all }t\geq0.
    \label{eq:app-tracked-convergence-conditions}
\end{equation}
The first inequality directly controls tracker excess risk relative to field energy; the second is a trajectory-wise slope inequality, with exponential and polynomial regimes at $q=1$ and $q>1$, respectively.

\begin{theorem}[Tracked population-flow convergence from excess risk]
    \label{thm:tracked-flow-convergence}
    Under the preceding assumptions, $\cF(P_t)\to0$ and hence $D_E(P_t,Q)\to0$.
    If $\cF(P_0)>0$, the quantitative bound is
    \begin{equation}
        \cF(P_t)
        \leq
        \begin{cases}
            \cF(P_0)\exp\!\left[-c(1-\gamma)t\right],
             & q=1, \\[3pt]
            \left[\cF(P_0)^{1-q}
                +c(q-1)(1-\gamma)t\right]^{-1/(q-1)},
             & q>1.
        \end{cases}
        \label{eq:app-tracked-flow-convergence}
    \end{equation}
\end{theorem}

\begin{proof}
    Since $\vv_t=-\nabla(\delta\cF/\delta P_t)$, the Wasserstein chain rule and Cauchy--Schwarz give
    \begin{equation}
        \frac{\dm}{\dm t}\cF(P_t)
        =
        -\langle\vv_t,u_t\rangle_{L^2(P_t)}
        \leq
        -(1-\gamma)\norm{\vv_t}_{L^2(P_t)}^2
        \leq
        -c(1-\gamma)\cF(P_t)^q.
    \end{equation}
    If $\cF(P_0)>0$, integrating this scalar differential inequality proves \eqref{eq:app-tracked-flow-convergence}.
    If $\cF(P_0)=0$, the same inequality and nonnegativity imply $\cF(P_t)=0$ for all $t\geq0$.
    In either case, properness and $D_E^2(P_t,Q)=2\cF(P_t)$ give the stated convergence.
\end{proof}

Broader quantile-space analyses characterize negative-distance MMD flows on the real line~\citep{duong2024distance}; the following elementary bound isolates a sufficient slope regime for tracked convergence.
\begin{proposition}[A concrete one-dimensional slope regime]
    \label{prop:one-dimensional-slope}
    Let $P_t$ and $Q$ be supported on a common finite interval $I\subset\R$, and suppose $P_t$ has a density $p_t\geq m>0$ almost everywhere on $I$, uniformly in $t$.
    Writing $F_t$ and $F_Q$ for their cumulative distribution functions, the exact energy-distance field satisfies
    \begin{equation}
        \norm{\vv_t}_{L^2(P_t)}^2
        =4\int_I(F_t-F_Q)^2p_t\,\dm x
        \geq4m\,\cF(P_t).
        \label{eq:app-one-dimensional-slope}
    \end{equation}
    Hence the slope condition in \thmref{thm:tracked-flow-convergence} holds with $q=1$ and $c=4m$; under the theorem's remaining assumptions,
    $\cF(P_t)\leq\cF(P_0)\exp[-4m(1-\gamma)t]$.
\end{proposition}

\begin{proof}
    The one-dimensional energy-distance identity gives
    $D_E^2(P_t,Q)=2\int_I(F_t-F_Q)^2\,\dm x$, and therefore
    $\cF(P_t)=\int_I(F_t-F_Q)^2\,\dm x$.
    Differentiating the first variation in \eqref{eq:first-variation-potential} at Lebesgue-almost every $x$ yields
    $\vv_t(x)=2(F_t(x)-F_Q(x))$.
    Integrating its square against $p_t$ and using $p_t\geq m$ proves \eqref{eq:app-one-dimensional-slope}; the rate then follows from \thmref{thm:tracked-flow-convergence}.
\end{proof}

\paragraph{Meaning for tracker training.}
The first condition in \eqref{eq:app-tracked-convergence-conditions} is a population-level certificate stated in the tracker's own regression geometry: full tracking requires its root excess risk to remain below the field norm, while mixing by $\rho<1$ attenuates tracker error.
The algorithm alone does not guarantee that a finite neural tracker maintains this relative certificate as the signal vanishes.
Absolute-error neighborhoods, smoothing, reweighting, and corridor conditioning are unified by \propref{prop:practical-mixed-field-stability} below.

\subsubsection{Finite-generator reference recursion}
\label{app:convergence-scope}

The results below analyze the stated first-order stochastic recursion and isolate how field-estimation error enters.
They do not cover AdamW momentum or preconditioning or the $\lambda<1$ corridor dynamics, and they do not establish the displayed bias and variance conditions for EMA-evolving quantities or the jointly trained tracker.

Suppose the parameterized objective $\cF(\mtheta)$ is lower bounded by $\cF_\star$ and $L$-smooth, with $L>0$, on a region containing every update segment between $\mtheta_t$ and $\mtheta_{t+1}$.
Let $\mathcal H_t$ be the history before iteration $t$, and absorb any fixed algorithmic loss scaling into the learning rate.
Write the parameter update as $\mtheta_{t+1}=\mtheta_t-\eta_t g_t$ and assume
$\Eb{g_t\mid\mathcal H_t}=\nabla\cF(\mtheta_t)+b_t$ and
$\Eb{\norm{g_t-\Eb{g_t\mid\mathcal H_t}}_2^2\mid\mathcal H_t}\leq\sigma^2$.
Here $b_t$ collects systematic field-estimation error, while $\sigma^2$ controls the residual stochastic variation.

\begin{proposition}[Optimization with finite field-estimation error]
    \label{prop:tracker-aware-optimization}
    For constant step size $\eta_t=\eta\leq 1/L$, the stochastic update satisfies
    \begin{equation}
        \min_{0\leq t<T}
        \Eb{\norm{\nabla_{\mtheta}\cF(\mtheta_t)}_2^2}
        \leq
        \frac{2(\cF(\mtheta_0)-\cF_\star)}{\eta T}
        +
        L\eta\sigma^2
        +
        \frac{1}{T}
        \sum_{t=0}^{T-1}
        \Eb{\norm{b_t}_2^2}.
        \label{eq:app-biased-sgd-stationarity}
    \end{equation}
\end{proposition}

\begin{proof}
    Let $a_t=\nabla\cF(\mtheta_t)$, $\bar g_t=\Eb{g_t\mid\mathcal H_t}=a_t+b_t$, and let $\E_t$ condition on $\mathcal H_t$.
    By $L$-smoothness,
    \begin{equation}
        \begin{aligned}
            \E_t[\cF(\mtheta_{t+1})]
             & \leq
            \cF(\mtheta_t)
            -\eta\langle a_t,\bar g_t\rangle
            +\frac{L\eta^2}{2}\E_t[\norm{g_t}_2^2]
            \\
             & \leq
            \cF(\mtheta_t)
            -\frac{\eta}{2}\norm{a_t}_2^2
            +\frac{\eta}{2}\norm{b_t}_2^2
            +\frac{L\eta^2\sigma^2}{2},
        \end{aligned}
    \end{equation}
    where the second inequality uses
    $-\langle a_t,\bar g_t\rangle=(\norm{b_t}_2^2-\norm{a_t}_2^2-\norm{\bar g_t}_2^2)/2$,
    $\E_t\norm{g_t}_2^2\leq\norm{\bar g_t}_2^2+\sigma^2$, and $\eta\leq1/L$.
    Taking expectations, summing over $t=0,\ldots,T-1$, and using $\cF(\mtheta_T)\geq\cF_\star$ proves \eqref{eq:app-biased-sgd-stationarity}.
\end{proof}

Under the local-gradient regularity conditions of \thmref{thm:three-body-equivalence}, exact bearings have $b_t=0$ relative to $\cF$; at $\lambda=1$, smoothed bearings likewise have $b_t=0$ relative to $\cF_\varepsilon$.
For the exact-bearing tracked algorithm at $\lambda=1$, let $e_t(\xx,\cc)=h_{\mphi,t}(\xx,\cc)-\vv_t(\xx,\cc)$.
The conditional expectation is $\vv_t+\rho e_t$, so the parameter-gradient bias in \propref{prop:tracker-aware-optimization} is
$b_t=-\rho\Eb{J_t^{\top}e_t\mid\mathcal H_t}$, where $J_t=\nabla_{\mtheta}\mg_{\mtheta_t}(\zz_t,\cc_t)$.
Jensen's inequality gives
$\norm{b_t}_2^2\leq\rho^2\Eb{\norm{J_t^{\top}e_t}_2^2\mid\mathcal H_t}$.
Combining this estimate with \eqref{eq:app-biased-sgd-stationarity} makes the Jacobian-weighted dependence explicit:
\begin{equation}
    \min_{0\leq t<T}
    \Eb{\norm{\nabla_{\mtheta}\cF(\mtheta_t)}_2^2}
    \leq
    \frac{2(\cF(\mtheta_0)-\cF_\star)}{\eta T}
    +L\eta\sigma^2
    +\frac{\rho^2}{T}
    \sum_{t=0}^{T-1}\Eb{\norm{J_t^{\top}e_t}_2^2}.
    \label{eq:app-tracker-error-stationarity}
\end{equation}
If $\norm{J_t}_{\mathrm{op}}\leq\kappa_J$ almost surely, the final term is further bounded by
$\rho^2\kappa_J^2T^{-1}\sum_t\Eb{\norm{e_t}_2^2}$.
Thus the Jacobian-weighted tracker error in \corref{cor:parameter-tracker-decomposition} controls the optimization neighborhood; under the operator-norm bound, the output-space excess risk in \thmref{thm:finite-tracker-decomposition} does as well.
An oracle tracker eliminates this term, but residual stochastic variance $\sigma^2$ may remain.
At $\lambda\neq1$, intra-source reweighting changes the reference objective from $\cF$ to $\cF_{\lambda}$, and query conditioning can change the mean direction even relative to $\cF_{\lambda}$.
Choosing $\cF_\varepsilon$ as the reference removes smoothing bias at $\lambda=1$; relative to the unsmoothed $\cF$, the field difference enters through $b_t$ when bounded.

Under explicit additional assumptions, this reference recursion also yields an expected convergence result.
Consider the exact-bearing $\lambda=1$ objective, assume the target is realizable, i.e., some $\mtheta^\star$ satisfies $\cF(\mtheta^\star)=0$, and suppose the parameterized objective satisfies
$\norm{\nabla\cF(\mtheta)}_2^2\geq2\mu\cF(\mtheta)$ with $\mu>0$ throughout the training region.
Assume also that $\norm{b_t}_2\leq\chi\norm{\nabla\cF(\mtheta_t)}_2$ almost surely for some $\chi<1$.
For tracked scattering, the stronger but directly interpretable condition
$\rho^2\Eb{\norm{J_t^\top e_t}_2^2\mid\mathcal H_t}\leq\chi^2\norm{\nabla\cF(\mtheta_t)}_2^2$
is sufficient by Jensen's inequality.
Let positive deterministic step sizes satisfy
$\eta_t\leq\min\{(1-\chi)/(L(1+\chi)^2),[\mu(1-\chi)]^{-1}\}$.

\begin{corollary}[Finite-generator convergence under PL and relative tracking]
    \label{cor:finite-generator-convergence}
    Under the preceding assumptions, if $\sum_t\eta_t=\infty$ and $\sum_t\eta_t^2<\infty$, then
    \begin{equation}
        \Eb{\cF(\mtheta_{t+1})\mid\mathcal H_t}
        \leq
        \left[1-\mu(1-\chi)\eta_t\right]\cF(\mtheta_t)
        +\frac{L\eta_t^2\sigma^2}{2}.
        \label{eq:app-finite-generator-convergence}
    \end{equation}
    Consequently, $\Eb{D_E^2(P_{\mtheta_t},Q)}\to0$.
\end{corollary}

\begin{proof}
    Write $a_t=\nabla\cF(\mtheta_t)$.
    The relative-bias condition gives
    $\langle a_t,a_t+b_t\rangle\geq(1-\chi)\norm{a_t}_2^2$ and
    $\norm{a_t+b_t}_2^2\leq(1+\chi)^2\norm{a_t}_2^2$.
    Applying $L$-smoothness as in the proof of \propref{prop:tracker-aware-optimization} and using the step-size bound therefore yields
    \begin{equation}
        \Eb{\cF(\mtheta_{t+1})\mid\mathcal H_t}
        \leq
        \cF(\mtheta_t)
        -\frac{(1-\chi)\eta_t}{2}\norm{a_t}_2^2
        +\frac{L\eta_t^2\sigma^2}{2}.
    \end{equation}
    The parameter-space PL inequality proves \eqref{eq:app-finite-generator-convergence}.
    After taking expectations, set $x_t=\Eb{\cF(\mtheta_t)}$, $a=\mu(1-\chi)$, and $d=L\sigma^2/2$.
    Then $x_t+d\sum_{k=t}^{\infty}\eta_k^2$ decreases by at least $a\eta_t x_t$; convergence of this sequence, the vanishing tail, and $\sum_t\eta_t=\infty$ force $x_t\to0$.
    Finally, $D_E^2(P_{\mtheta_t},Q)=2\cF(\mtheta_t)$.
\end{proof}
The same proof applies to $\cF_\varepsilon$ under the corresponding smoothness, realizability, PL, and relative-tracking assumptions, yielding $\Eb{\cF_\varepsilon(\mtheta_t)}\to0$.

\paragraph{Interpretation.}
The unconstrained exact and smoothed $\lambda=1$ energies have the unique optimum $P=Q$; \thmref{thm:tracked-flow-convergence} and its smoothed counterpart show that sufficiently accurate tracking preserves convergence whenever the corresponding flow satisfies the stated slope condition.
A finite generator instead restricts the reachable distributions to $\{(\mg_{\mtheta})_\#p(\zz)\}$, and SGD optimizes a nonconvex parameterization; finite tracking, corridor conditioning, intra-source reweighting, and the choice of comparison geometry separate deployment from the exact full-space flow.
Accordingly, within this reference recursion, \eqref{eq:app-biased-sgd-stationarity} gives a finite-error stationarity bound, while \corref{cor:finite-generator-convergence} obtains expected energy-distance convergence only after adding realizability, parameter-space PL geometry, relative tracker accuracy, and diminishing step sizes.

\subsection{Additional Estimator Properties}

\subsubsection{Source-law necessity and variance optimality}
\label{app:minimal-three-body}
Fix a projectile $\xx_{\mathrm p}$ and let
$\beta_{\xx_{\mathrm p}}(\yy)=(\yy-\xx_{\mathrm p})/\norm{\yy-\xx_{\mathrm p}}$ away from coincidence and zero at coincidence.
For independent $\xx_{\mathrm r}\sim Q$ and $\xx_{\mathrm s}\sim P$, define
$T=\beta_{\xx_{\mathrm p}}(\xx_{\mathrm r})-\beta_{\xx_{\mathrm p}}(\xx_{\mathrm s})$ and
$\vv_P(\xx_{\mathrm p})=\Eb{T}$.
Consider square-integrable source-sampling rules whose measurable form and auxiliary-randomness law do not depend on the unknown $(P,Q)$; call such a rule universally unbiased when its expectation equals $\vv_P(\xx_{\mathrm p})$ for every $(P,Q)$.

\begin{proposition}[Source-law necessity and Rao--Blackwell optimality]
    \label{prop:triplet-minimality-optimality}
    No universally unbiased rule can omit either source law entirely.
    Among rules observing one independent draw from each law, $T$ is the unique deterministic universally unbiased rule and minimizes covariance in the positive-semidefinite order, hence also mean-squared error, among randomized universally unbiased rules.
    For the average $\bar T_K$ of $K$ independent copies,
    \begin{equation}
        \Eb{\norm{\bar T_K-\vv_P(\xx_{\mathrm p})}_2^2}
        =\frac{\operatorname{tr}\operatorname{Cov}_Q(\beta_{\xx_{\mathrm p}})
            +\operatorname{tr}\operatorname{Cov}_P(\beta_{\xx_{\mathrm p}})}{K}
        \leq\frac2K.
        \label{eq:app-triplet-optimal-variance}
    \end{equation}
\end{proposition}

\begin{proof}
    Without a draw from $P$ (respectively $Q$), a rule's expectation cannot vary with that law, and two point masses with different bearings contradict universal unbiasedness.
    For any randomized one-draw-per-law rule $A$, taking $Q=\delta_{\rr}$ and $P=\delta_{\ss}$ forces
    $\Eb{A\mid\xx_{\mathrm r},\xx_{\mathrm s}}=T$ pointwise.
    Hence
    $\operatorname{Cov}(A)=\operatorname{Cov}(T)+\Eb{\operatorname{Cov}(A\mid\xx_{\mathrm r},\xx_{\mathrm s})}$, proving the optimality and deterministic uniqueness claims.
    Independence gives the variance identity, averaging supplies $1/K$, and each covariance trace is at most one because a bearing has norm at most one.
\end{proof}
Also $\norm{T}_2\leq2$ almost surely.
With zero collision probability, the numerator in \eqref{eq:app-triplet-optimal-variance} equals
$2-\norm{\EEb{\xx_{\mathrm r}\sim Q}{\beta_{\xx_{\mathrm p}}(\xx_{\mathrm r})}}_2^2
    -\norm{\EEb{\xx_{\mathrm s}\sim P}{\beta_{\xx_{\mathrm p}}(\xx_{\mathrm s})}}_2^2$.
The source-law statement concerns universal dependence on both laws, not whether every randomized realization must query both.

\subsection{Implementation and Interpretive Derivations}

\subsubsection{Denominator smoothing as a proper objective}
\label{app:denominator-smoothing}
Related work constructs differentiable approximations of negative-distance kernels for Wasserstein gradient flows~\citep{rux2026smoothed}; here we analyze the denominator form used in our implementation.
For $\varepsilon>0$, define
$\psi_\varepsilon(r)=r-\varepsilon\log(1+r/\varepsilon)$ and let $\cF_\varepsilon$ be \eqref{eq:app-energy-functional} with every distance $r$ replaced by $\psi_\varepsilon(r)$.

\begin{proposition}[Proper objective induced by denominator smoothing]
    \label{prop:denominator-smoothing}
    Wherever the first variation is valid, the denominator-smoothed $\lambda=1$ field satisfies
    \begin{equation}
        \vv_\varepsilon(\xx)
        =\EEb{\xx_{\mathrm r}\sim Q}{\frac{\xx_{\mathrm r}-\xx}{\norm{\xx_{\mathrm r}-\xx}_2+\varepsilon}}
        -\EEb{\xx_{\mathrm s}\sim P}{\frac{\xx_{\mathrm s}-\xx}{\norm{\xx_{\mathrm s}-\xx}_2+\varepsilon}}
        =-\nabla_{\xx}\frac{\delta\cF_\varepsilon}{\delta P}(\xx).
        \label{eq:app-smoothed-potential-field}
    \end{equation}
    Moreover, $\cF_\varepsilon(P)\geq0$ and, under finite first moments, equality holds if and only if $P=Q$.
\end{proposition}

\begin{proof}
    Since $\psi_\varepsilon'(r)=r/(r+\varepsilon)$,
    $\nabla_{\xx}\psi_\varepsilon(\norm{\xx-\yy}_2)=(\xx-\yy)/(\norm{\xx-\yy}_2+\varepsilon)$, with zero gradient at coincidence.
    Applying this identity to the first variation proves \eqref{eq:app-smoothed-potential-field}; the frozen-target pullback gives the corresponding local parameter gradient.

    For properness, write $\psi_\varepsilon(\sqrt t)=f_\varepsilon(t)$.
    Because $e^{-u\sqrt t}$ is completely monotone in $t$, the identity $(\sqrt t+\varepsilon)^{-1}=\int_0^\infty e^{-u\varepsilon}e^{-u\sqrt t}\,\dm u$ shows that $f_\varepsilon'(t)=1/[2(\sqrt t+\varepsilon)]$ is also completely monotone.
    Since $f_\varepsilon(0)=0$, $f_\varepsilon(t)/t\to0$ as $t\to\infty$, and $f_\varepsilon$ is nonconstant, the Bernstein-function representation theorem~\citep{schilling2012bernstein} gives a L\'evy--Khintchine representation with neither constant nor linear term:
    $f_\varepsilon(t)=\int_{(0,\infty)}(1-e^{-st})\,\nu_\varepsilon(\dm s)$ for a nonzero positive L\'evy measure $\nu_\varepsilon$.
    Consequently,
    $2\cF_\varepsilon(P)=\int_{(0,\infty)}\mathrm{MMD}_{k_s}^2(P,Q)\,\nu_\varepsilon(\dm s)$ with Gaussian kernels $k_s(\xx,\yy)=e^{-s\norm{\xx-\yy}_2^2}$.
    Each $k_s$ is characteristic for $s>0$, proving nonnegativity and identity of indiscernibles; finite first moments suffice because $0\leq\psi_\varepsilon(r)\leq r$.
\end{proof}
The map $\aa\mapsto\aa/(\norm{\aa}_2+\varepsilon)$ is globally $1/\varepsilon$-Lipschitz.
Consequently, $\vv_\varepsilon$ is $2/\varepsilon$-Lipschitz in its spatial argument and $1/\varepsilon$-Lipschitz in $P$ under $W_1$; for fixed $Q$, the associated McKean--Vlasov characteristic dynamics are therefore well posed under finite first moments.
The smoothed source estimator remains conditionally unbiased for $\vv_\varepsilon$; hence the frozen-target, tracker-decomposition, and tracked-convergence arguments apply with $(\cF,\vv)$ replaced by $(\cF_\varepsilon,\vv_\varepsilon)$ under the corresponding chain-rule and slope assumptions.
Replacing the model--model coefficient by $\lambda$ gives the analogous smoothed $\lambda$-weighted functional; before corridor conditioning, its negative-gradient field equals the corresponding source mean.
The properness claim above is specific to $\lambda=1$.
Away from coincidence, the exact and smoothed source bearings differ in norm by
$\varepsilon/(\norm{\yy-\xx}_2+\varepsilon)$; both are zero at coincidence. Therefore
\begin{equation}
    \norm{\vv_\varepsilon(\xx)-\vv(\xx)}_2
    \leq\EEb{\xx_{\mathrm r}\sim Q}{\frac{\varepsilon\mathbf{1}\{\xx_{\mathrm r}\neq\xx\}}
        {\norm{\xx_{\mathrm r}-\xx}_2+\varepsilon}}
    +\EEb{\xx_{\mathrm s}\sim P}{\frac{\varepsilon\mathbf{1}\{\xx_{\mathrm s}\neq\xx\}}
        {\norm{\xx_{\mathrm s}-\xx}_2+\varepsilon}}
    \longrightarrow0
    \quad(\varepsilon\downarrow0).
    \label{eq:app-smoothed-bearing-bias}
\end{equation}
The convergence follows by dominated convergence; if both source distances are at least $\delta>0$, the bound is at most $2\varepsilon/(\delta+\varepsilon)$.
Thus smoothing retains a proper objective while changing its geometry, and its field converges pointwise to the exact bearing field as $\varepsilon\downarrow0$.

\subsubsection{Near-endpoint surrogate control}
\label{app:near-endpoint-control}
Let $\psi$ be either the exact potential $\psi(r)=r$ or the smoothed potential $\psi_\varepsilon$ above, and define
$\beta_\psi(\aa)=\psi'(\norm{\aa}_2)\aa/\norm{\aa}_2$ away from zero and $\beta_\psi(0)=0$.
In both cases $\norm{\beta_\psi}_2\leq1$, and
\begin{align*}
    D_\psi^2(P,Q)
     & =2\Eb{\psi(\norm{\xx-\yy}_2)}
    -\Eb{\psi(\norm{\xx-\xx'}_2)}
    -\Eb{\psi(\norm{\yy-\yy'}_2)},   \\
    \cF_{\psi,\lambda}(P;Q)
     & =\Eb{\psi(\norm{\xx-\yy}_2)}
    -\frac{\lambda}{2}\Eb{\psi(\norm{\xx-\xx'}_2)},
\end{align*}
where $\xx,\xx'\sim P$ and $\yy,\yy'\sim Q$ are independent.
The discrepancy $D_\psi^2$ is proper for either choice of $\psi$.
Write
$\vv_{\psi,\lambda}(\xx)=\EEb{\yy\sim Q}{\beta_\psi(\yy-\xx)}-\lambda\EEb{\xx'\sim P}{\beta_\psi(\xx'-\xx)}$
for the corresponding source-mean field before tracker conditioning.

\paragraph{The instant attraction-only endpoint.}
For $\lambda=0$, define
$f_{\psi,Q}(\xx)=\EEb{\yy\sim Q}{\psi(\norm{\xx-\yy}_2)}$ and
$M_{\psi}(Q)=\operatorname*{arg\,min}_{\xx\in\R^d}f_{\psi,Q}(\xx)$.
If $Q$ has a finite first moment, continuity and coercivity make $M_\psi(Q)$ nonempty for either admissible potential, and
\begin{equation*}
    \cF_{\psi,0}(P;Q)=\int f_{\psi,Q}(\xx)\,P(\dm\xx)
\end{equation*}
is minimized exactly by laws supported on $M_\psi(Q)$.
For $\psi(r)=r$, this is the geometric-median set; if it is a singleton $\{\xx^\star\}$, the unique minimizer is $\delta_{\xx^\star}$ rather than $Q$, unless $Q$ is already that point mass.
Thus the instant $\rho=0,\lambda=0$ objective is generally not proper without generated-source repulsion.
This characterization concerns the functional before corridor conditioning and does not characterize the tracked $\rho>0$ update.

\begin{proposition}[Controlled intra-source reweighting]
    \label{prop:near-endpoint-reweighting}
    Before corridor conditioning, the source-mean field $\vv_{\psi,\lambda}$ obeys
    $\norm{\vv_{\psi,\lambda}(\xx)-\vv_{\psi,1}(\xx)}_2\leq1-\lambda$.
    Moreover, if $\delta\geq0$ and
    $\cF_{\psi,\lambda}(P;Q)\leq\cF_{\psi,\lambda}(Q;Q)+\delta$, then
    \begin{equation}
        D_\psi^2(P,Q)
        \leq
        (1-\lambda)\Eb{\psi(\norm{\yy-\yy'}_2)}+2\delta.
        \label{eq:app-near-endpoint-objective-control}
    \end{equation}
\end{proposition}

\begin{proof}
    The field difference is $(1-\lambda)\EEb{\xx'\sim P}{\beta_\psi(\xx'-\xx)}$, proving the uniform bound.
    Also,
    $\cF_{\psi,\lambda}(P;Q)=\frac12D_\psi^2(P,Q)+\frac12\Eb{\psi(\norm{\yy-\yy'}_2)}+\frac{1-\lambda}{2}\Eb{\psi(\norm{\xx-\xx'}_2)}$.
    Comparing this identity at $P$ and $Q$ and dropping the nonnegative model--model term proves \eqref{eq:app-near-endpoint-objective-control}.
\end{proof}
Consequently, if $\lambda_n\uparrow1$, $\delta_n\to0$, and
$\cF_{\psi,\lambda_n}(P_n;Q)\leq\cF_{\psi,\lambda_n}(Q;Q)+\delta_n$, then $D_\psi(P_n,Q)\to0$; this includes global minimizers whenever they exist.

For the full mixed target, let $Y_{\psi,\lambda}$ denote the instant vector, retain the query $\mathbf q$ and history $\mathcal H$ from \secref{app:tracker-conditional-expectation}, and write
$m_\lambda(\mathbf q)=\Eb{Y_{\psi,\lambda}\mid\mathbf q,\mathcal H}$ and
$\bar\vv_{\psi,\lambda}(\xx)=\Eb{Y_{\psi,\lambda}\mid\xx,\mathcal H}=\vv_{\psi,\lambda}(\xx)$.
Define the tracker excess risk
$E_{\mphi}=\Eb{\norm{h_{\mphi}(\mathbf q)-m_\lambda(\mathbf q)}_2^2\mid\mathcal H}$ and the oracle corridor mismatch
$C_\lambda=\norm{\Eb{m_\lambda(\mathbf q)\mid\xx,\mathcal H}-\bar\vv_{\psi,\lambda}(\xx)}_{L^2(P)}$.
Let $u_\rho(\xx)=\Eb{G_\rho\mid\xx,\mathcal H}$ be the mean mixed field.

\begin{proposition}[Stability of the deployed mixed field]
    \label{prop:practical-mixed-field-stability}
    The deviation from the proper endpoint field satisfies
    \begin{equation}
        \norm{u_\rho-\vv_{\psi,1}}_{L^2(P)}
        \leq
        \underbrace{(1-\lambda)+\rho\left(C_\lambda+\sqrt{E_{\mphi}}\right)}_{\displaystyle \Delta}.
        \label{eq:app-practical-mixed-field-bias}
    \end{equation}
\end{proposition}

\begin{proof}
    The exact decomposition is
    $u_\rho-\vv_{\psi,1}=(\bar\vv_{\psi,\lambda}-\vv_{\psi,1})+\rho\Eb{h_{\mphi}(\mathbf q)-m_\lambda(\mathbf q)\mid\xx,\mathcal H}+\rho(\Eb{m_\lambda(\mathbf q)\mid\xx,\mathcal H}-\bar\vv_{\psi,\lambda})$.
    Apply conditional $L^2$ contraction and \propref{prop:near-endpoint-reweighting}.
\end{proof}
At $\lambda=1$, the query is the projectile and $C_1=0$; for $\lambda<1$, $C_\lambda$ isolates the bias caused by conditioning on the corridor rather than the projectile.
Along a flow, let $\Delta_t$ denote the right-hand side of \eqref{eq:app-practical-mixed-field-bias} and write $\cF_\psi(P)=\frac12D_\psi^2(P,Q)$.
Under the corresponding chain rule and linear slope $\norm{\vv_{\psi,1,t}}_{L^2(P_t)}^2\geq c\cF_\psi(P_t)$, Young's inequality yields
\begin{equation}
    \cF_\psi(P_t)
    \leq
    e^{-ct/2}\cF_\psi(P_0)
    +\frac12\int_0^t e^{-c(t-s)/2}\Delta_s^2\,\dm s.
    \label{eq:app-practical-flow-stability}
\end{equation}
Hence bounded $\Delta_t\to0$ gives $D_\psi(P_t,Q)\to0$, while $\sup_t\Delta_t\leq\bar\Delta$ gives
$\limsup_{t\to\infty}\cF_\psi(P_t)\leq\bar\Delta^2/c$.
If instead $\Delta_t\leq\gamma\norm{\vv_{\psi,1,t}}_{L^2(P_t)}$ for some $\gamma<1$, the argument of \thmref{thm:tracked-flow-convergence} gives exact convergence under its remaining assumptions.
Thus \eqref{eq:app-near-endpoint-objective-control} controls the weighted objective, while \eqref{eq:app-practical-flow-stability} quantifies stability of the deployed flow around the proper endpoint.

\subsubsection{Line-integral interpretation of the displacement tracker}
\label{app:gan-like-line-integral}
For the modified displacement analogue in \secref{sec:path-conditioned-fallback}, let
$\mv_{\mphi}=\nabla u_{\mphi}$,
$\Delta=\xx_{\mathrm r}-\xx_{\mathrm p}$, and
$\xx_\alpha=(1-\alpha)\xx_{\mathrm p}+\alpha\xx_{\mathrm r}$.
The fundamental theorem for line integrals gives
\begin{equation}
    \int_0^1\nabla_{\xx}u_{\mphi}(\xx_\alpha,\cc)^\top\Delta\,\dm\alpha
    =u_{\mphi}(\xx_{\mathrm r},\cc)-u_{\mphi}(\xx_{\mathrm p},\cc).
    \label{eq:app-gan-like-line-integral}
\end{equation}
Expanding $\Eb{\norm{\nabla u_{\mphi}(\xx_\alpha,\cc)-\Delta}_2^2}$ for
$\alpha\sim\mathcal U(0,1)$ and applying this identity gives \eqref{eq:gan-like-critic-objective}; the remaining $\Eb{\norm{\Delta}_2^2}$ is constant for fixed generator parameters.
Hence displacement regression maximizes the real--generated potential gap minus one half of the corridor-averaged squared gradient norm.
The infimum over conservative fields is characterized by the $L^2$ projection of $\Eb{\Delta\mid\xx_\alpha,\cc}$ onto their closure under the corridor measure; a finite network is further restricted by its parameterization.
This is not an identity for the deployed unit-bearing update, and a general vector tracker retains the conditional-expectation decomposition but not the scalar line-integral interpretation.

\end{document}